\documentclass[journal]{IEEEtran}
\IEEEoverridecommandlockouts

\def\BibTeX{{\rm B\kern-.05em{\sc i\kern-.025em b}\kern-.08em
    T\kern-.1667em\lower.7ex\hbox{E}\kern-.125emX}}

\usepackage{soul}
\usepackage[normalem]{ulem}
\usepackage{graphicx}
\usepackage{epstopdf}
\usepackage{amsmath,amssymb,amsfonts}
\usepackage{algorithmic}
\usepackage{textcomp}
\usepackage{lipsum}
\usepackage{epstopdf}
\usepackage{wrapfig}
\usepackage{hyperref}       
\usepackage{url}            
\usepackage{epsfig,graphics}
\usepackage{algorithm}
\usepackage{color}
\usepackage{subcaption}
\usepackage{amsmath} 
\usepackage{amsmath} \DeclareMathOperator*{\argmin}{arg\,min}
\newtheorem{theorem}{Theorem}
\newtheorem{lemma}{Lemma}
\newtheorem{proof}{Proof}
\newcommand{\etal}{{\em et al\,.}}       
\newcommand{\eg}{{\em e.g.}}           
\newcommand{\ie}{{\em i.e.}}           
\newcommand{\etc}{{\em etc.}}         

\newcommand\redsout{\bgroup\markoverwith{\textcolor{red}{\rule[0.5ex]{2pt}{0.4pt}}}\ULon}

\newcommand{\yh}[1]{\textcolor{black}{#1}}
\newcommand{\mr}[1]{\textcolor{black}{#1}}
\newcommand{\yhk}[1]{\textcolor{black}{#1}}

\ifCLASSINFOpdf
\else
\fi

\begin{document}
\title{Learning-based Computer-aided Prescription Model for Parkinson's Disease: A Data-driven Perspective}

\author{Yinghuan~Shi,
        Wanqi~Yang,
        Kim-Han Thung,
        Hao~Wang,
        Yang~Gao,
        Yang~Pan,
        Li~Zhang,
        Dinggang~Shen
\thanks{Yinghuan Shi and Yang Gao are with the State Key Laboratory for Novel Software Technology, Nanjing University, China. They are also with National Institute of Healthcare Data Science, Nanjing University, China. E-mail: syh@nju.edu.cn, gaoy@nju.edu.cn.}
\thanks{Wanqi Yang is with the School of Computer Science, Nanjing Normal University, China. E-mail: yangwq@njnu.edu.cn.}
\thanks{Kim-Han Thung is with the Department of Radiology and BRIC, UNC Chapel Hill, US. E-mail: khthung@med.unc.edu.}
\thanks{Hao Wang is with Inception Institute of Artificial Intelligence, UAE. E-mail: wanghao.hku@gmail.com}
\thanks{Yang Pan and Li Zhang are with the Nanjing Medical University and Nanjing Brain Hospital, China. E-mail: neuropanyang@163.com and neurozhangli@163.com.}
\thanks{Dinggang Shen is with Department of Research and Development, Shanghai United Imaging Intelligence Co., Ltd., Shanghai, China, and also the Department of Brain and Cognitive Engineering, Korea University, Korea. E-mail: Dinggang.Shen@gmail.com.}
\thanks{This work is supported by National Key Research and Development Program of China (2019YFC0118300), Natural Science Foundation of China (61673203, 61603193), and Jiangsu Provincial Key Research and Development Project (BE2018610).}
}

\markboth{}%
{Shell \MakeLowercase{\textit{et al.}}: Bare Demo of IEEEtran.cls for IEEE Journals}

\maketitle
\begin{abstract}
In this paper, we study a novel problem: \yh{\lq\lq\emph{automatic prescription recommendation for PD patients.}\rq\rq}
\yh{To realize this goal, we first build a dataset by collecting 1) symptoms of PD patients, and 2) their prescription drug provided by neurologists.
Then, we build a novel computer-aided prescription model by learning the relation between observed symptoms and prescription drug. Finally, for the new coming patients, we could recommend (predict) suitable prescription drug on their observed symptoms by our prescription model.}
From the methodology part, our proposed model, namely Prescription viA Learning lAtent Symptoms (PALAS), could recommend prescription using the multi-modality representation of the data. In PALAS, a latent symptom space is learned to better model the relationship between symptoms and prescription drug, as there is a large semantic gap between them. Moreover, we present an efficient alternating optimization method for PALAS. We evaluated our method using the data collected from 136 PD patients \yh{at Nanjing Brain Hospital, which can be regarded as a large dataset in PD research community.}
The experimental results demonstrate the effectiveness and clinical potential of our method in this recommendation task, if compared with other competing methods.
\end{abstract}


\IEEEpeerreviewmaketitle

\section{Introduction}
\IEEEPARstart{P}{arkinson's} Disease (PD)\footnote{http://www.parkinson.org/understanding-parkinsons/10-early-warning-signs} is a chronic and progressive neurological disorder that associated with symptoms such as trouble of moving and sleeping, tremor, shaking, dizziness, and fainting. It is now believed as the second most common neurodegenerative disease, affecting nearly one million people in US \cite{Hirschauer2015}. Though researches show that PD might be related to genetic and environmental factors, the exact cause of PD is still remained unknown \cite{SAMII20041783}.

 \begin{figure}[htbp]
 \centering
 \includegraphics[width=3.4in]{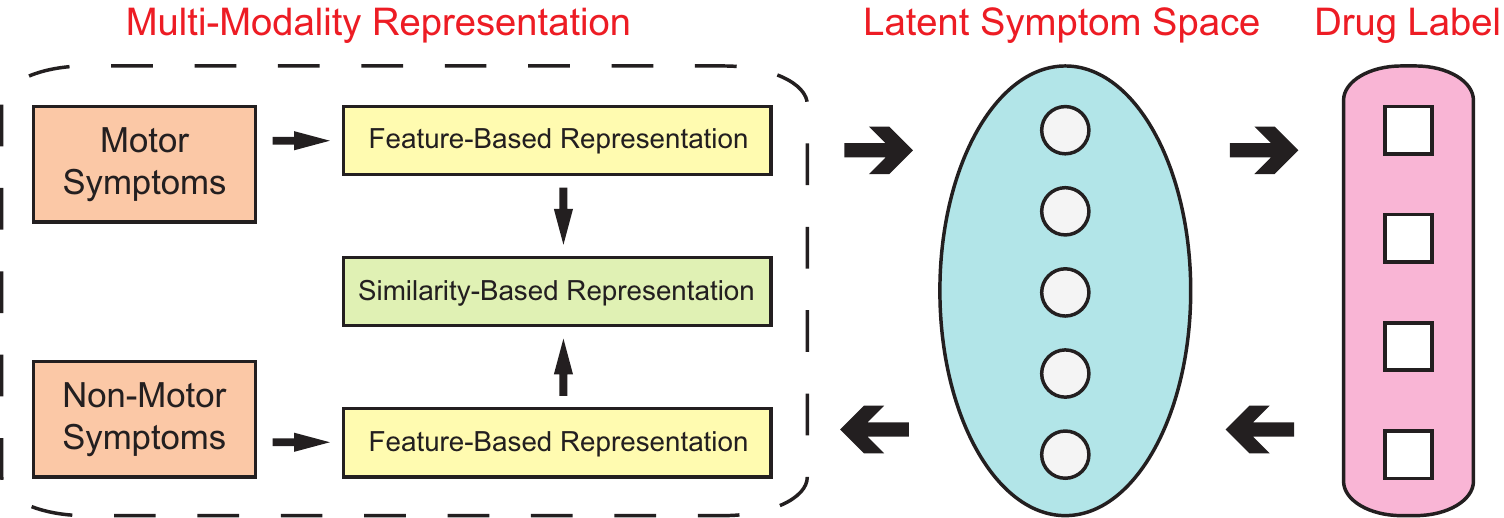}
 \caption{The proposed PALAS model.}\label{framework}
 \end{figure}

In recent years, aiming to guide better intervention therapy, computer-aided methods for PD patients have aroused lots of research interests \cite{Adeli2016206}\cite{Cai19JBHI}\cite{Cara18JBHI}\cite{Cor19JBHI}\cite{Hirschauer2015}\cite{Lei19JBHI}\cite{Limiccai17}\cite{Liu2016}\cite{Oliveira15}\cite{Singh201530}. All these researches are focusing on \emph{computer-aided diagnosis}, which is used to predict whether a patient belongs to PD in the early stage.
\mr{To the best of our knowledge, the study of \emph{computer-aided prescription} for PD has seldom been touched, whose goal is to choose the suitable treatment drugs for PD patients based on their observed symptoms.}
Motivated by the success application of artificial intelligence in various studies, we in this paper investigate a novel problem: \yh{\lq\lq\textbf{\emph{automatically recommend the suitable drugs for PD patients according to their observed symptoms.}}\rq\rq}

This prescription recommendation (or prediction) task is clinically meaningful and practically feasible. In conventional human-based prescription, we require professional experience of a \yh{neurologist} to prescribe medication for a PD patient. Unfortunately, \mr{medical treatment for PD is very much personalized.} In reality, different \yh{neurologists} usually prescribe different treatment drugs according to their own experience and judgement. Furthermore, if other factors such as prices and hidden side effects of drugs are considered, seeking a (near) optimal prescription for each individual patient would become even more complicated and difficult. Thus, a data-driven drug prescription, generated by machine learning technique that incorporates previous human-based prescription results, will be very helpful, and can be used as a guidance or reference, especially when there is a large diversity of views among different \yh{neurologists}.

To realize this goal, \emph{firstly}, we have spent two years to collect the motor and non-motor symptoms from 136 PD patients, which could be regarded as a large population in PD research community. Basically, the motor symptoms record the daily body-movement of the subjects, while the non-motor symptoms describe the mental states of the subjects, such as impression, psychology, \etc
All these PD patients also have their symptoms improved after clinical treatment, where the drugs used in the successive treatment are recorded. These information could be used to evaluate the performance of computer-aided prescription using the human-based prescriptions as the ground truth. \footnote{\yh{Note that in Figure \ref{framework}, drug label is a binary matrix with rows as patients and columns as drugs to record if current patient takes a certain drug (\ie, 1 indicates yes and 0 indicates no).}}

\emph{Secondly}, using the recorded data from the observed motor and non-motor symptoms, we combine feature-based and similarity-based representations to obtain a multi-modality representation.

\emph{Finally}, for prediction, we propose a novel model called Prescription viA Learning lAtent Symptoms (PALAS) (see Figure \ref{framework}), which is inspired and motivated by the following observations:
\begin{itemize}
\item \textbf{Sparse drug label}. One single type of drug can relieve only a small amount of all symptoms. Also, it is very rare to have a patient that is prescribed with all kinds of treatment drugs. Instead, most patients would only take a very few from all the candidates of treatment drugs, making the drug labels for each subject being sparse.
\item \textbf{Large semantic gap}. There indeed exists a large semantic gap between the observed symptoms and the treatment drugs which requires professional experience from a clinician to understand. Therefore, the relation of symptoms and drug labels should not be modeled directly as it may limit the performance of the prediction model.
\end{itemize}

Our proposed model PALAS is able to address the above mentioned issues by learning a \emph{latent symptom space} that exploits the intrinsic symptom-to-drug connection.
Specifically, based on the latent symptom space, we introduce a truncated objective to jointly learn (1) the observed symptoms -to- latent symptoms connection and (2) latent symptoms -to- drug labels connection. Also, we impose a sparse regularizer as a constraint to guarantee the sparsity of drug label.
Moreover, considering the limited number of subjects in our study, we employ a transductive setting to incorporate all the valuable data to train a better model. In other words, we also use the observed features (not including labels) of the testing data in PALAS.
We solve PALAS effectively using an alternating optimization algorithm, and analyze its convergence.

Our contributions in this paper are summarized as follows:
\textbf{1)} Our work is the first attempt to combine the observed motor and non-motor symptoms to recommend/predict the suitable treatment drugs for PD patient as a data-driven reference.
\textbf{2)} To leverage various observed symptoms, we propose to use multi-modality representation that incorporates both the feature-based and similarity-based representations of the observed symptoms.
\textbf{3)} To address the large semantic gap between the observed symptoms and the prescription drugs, as well as to better capture the intrinsic symptom-to-drug relationship, we propose PALAS model that learns a latent symptom space linking the observed symptoms and the prescription drugs.

\section{Related Works}
\label{sec:relatedwork}
Recently, several works have been developed for computer-aided PD diagnosis. Hirschauer \etal \cite{Hirschauer2015} employed an enhanced probabilistic neural network that integrated various features to diagnose whether a subject belongs to PD or healthy control (HC). Oliveira \etal \cite{Oliveira15} used voxel-wise features extracted from the SPECT images for PD/HC prediction. Using MRI data after registration \cite{ADL1}, Singh \etal \cite{Singh201530} combined self-organizing map and least squares SVM for PD diagnosis. Liu \etal \cite{Liu2016} performed iterative canonical correlation analysis to analyze different brain regions through T1-weighted MR images. Adeli \etal \cite{Adeli2016206} distinguished PD and HC using MRI data with a joint feature-sample selection method. Lei \etal \cite{Lei19JBHI} proposed a multitask learning-based framework to model the relation between multiple modalities (\eg, sample, feature, and clinical scores) for PD diagnosis. Also, a deep learning-based multimodal assessment of PD was developed in \cite{Cor19JBHI}. \yh{In \cite{AD2}, a method of using machine learning tools on voice signal was introduced for Parkinson's disease detection. Also, a mobile application was developed in \cite{AD3} to detect and identify the motor disorders.}
To assess brain connectivity dynamics and investigate the connectivity way, Cai \etal \cite{Cai19JBHI} introduced the dynamic graph as a theoretical tool to model and analyze the functional connectivity in PD.
All these aforementioned methods were used to address the PD diagnostic problem, while the computer-aided prescription was rarely addressed, if it has ever been explored before.


Moreover, as for another representative neurodegenerative disease, Alzheimer's disease (AD), several computer-aided methods have been developed in recent years. The common goal of these methods is to learn a common space from multi-view/modal data to aid the diagnosis. For example, in \cite{zhang11}, a multi-view algorithm was introduced by automatically learning the best combination of multiple kernels.
In \cite{Zhang17tkde}, Zhang \etal proposed to use a shared tree-structured model to investigate the genetic risk factors for AD treatment.
In \cite{Jie18TIP}, a method was developed to combine individual sub-networks of brain connectivity for MCI diagnosis.
Also, Xu \etal \cite{Xu17ijcai} utilized a sparse model under low-rank constraint for cognitive assessment prediction for AD patients.
Kim \etal \cite{ComKimHBM18} presented a multi-modal extreme learning machine to integrate different imaging modalities.
Shi \etal \cite{Shi2014Joint}\cite{Shi2020Joint} proposed to use the coupled features and coupled boosting for multi-view data analysis for AD diagnosis.
Tong \etal \cite{ComTongPR17} analyzed complementary relation between multiple modalities and present a diagnosis model for AD.
The incomplete neuroimaging and genetic data was combined for AD diagnosis in \cite{ADL4}.
Please note that, these methods for AD focus on the binary or small-sized classification problem while our method deals with multi classification problem (\ie, number of class is 31) which is more challenging. In addition, several methods \cite{ADL2}\cite{ADL3} for brain activity analysis are also related to our works.

Furthermore, from the methodology perspective, our study is related to multi-view learning methods, \eg, co-training style algorithms \cite{Blum:1998:CLU:279943.279962}\cite{Sindhwani:2008:RML:1390156.1390279}, multi-kernel learning methods \cite{Bach:2004:MKL:1015330.1015424}\cite{Rakotomamonjy_JMLR_2008/LIDIAP}, and subspace learning methods \cite{Hardoon:2009:CAK:1487444.1487460}\cite{Quadrianto2011Learning}, which learn the intrinsic or consensus representation from different views of data (or modalities) \cite{XuTao13}.
Also, from the perspective of selecting the (near) optimal drugs, our task is related to multi-label learning where each sample is associated with a set of labels, such as in problem transformation \cite{BoutellCR20041757}\cite{Furnkranz:2008:MCV:1416930.1416934}, and algorithm adaptation \cite{Clare2002Knowledge}\cite{Elisseeff01akernel}\cite{Zhang:2007:MLL:1234417.1234635}.

Indeed, among these works, our proposed PALAS model is closest to \cite{Zhangzhongfei2012}\cite{Zou2016}, which considered multi-modality features and multiple labels simultaneously. In particular, Fang \etal presented a method \cite{Zhangzhongfei2012} to learn a multi-modality consensus representation for image classification, which imposes a maximum margin criterion to improve the generalization ability. For image annotation task, Zou \etal developed a method called MVML \cite{Zou2016} by combining multiple single-label learners in a boosting framework with the base learner as a modified multi-modality based SVM. However, the prediction performance is very limited if we directly utilize \cite{Zhangzhongfei2012}\cite{Zou2016} for our study as (1) the intrinsic symptom-to-drug connection is not considered, and (2) the sparse property of drug label matrix is ignored. We will further discuss this using our experimental results  in Section \ref{sec:experiments}.

In brief, our problem is novel as it is fundamentally different from the previous computer-aided PD diagnosis works. In addition, the related works in multi-view learning and multi-label learning could not be efficiently applied to our problem.

\begin{figure}[htbp]
\centering
\includegraphics[width=2.3in]{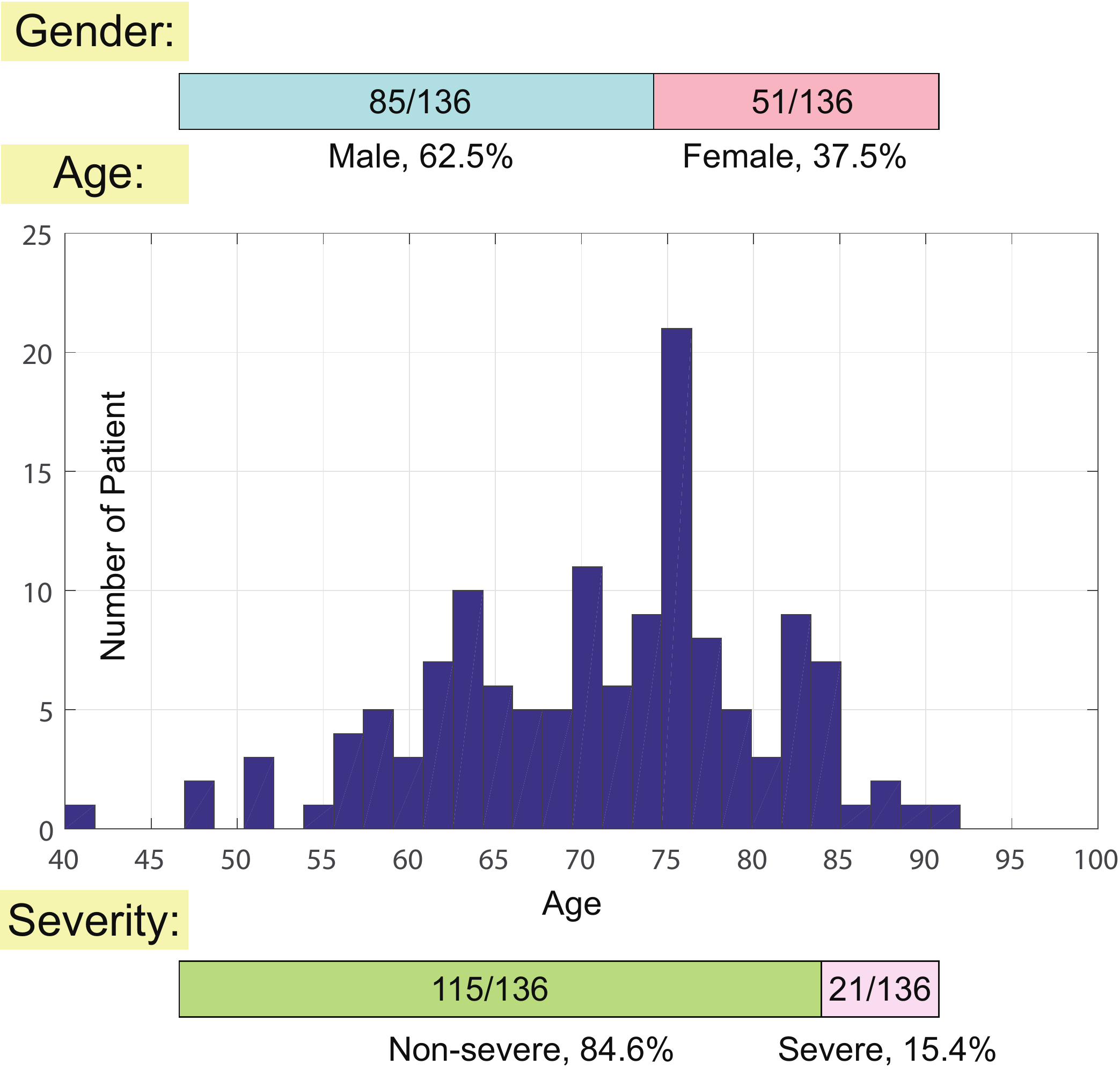}
\caption{\yh{The demographic information of all the PD patients.}}\label{fig_age}
\vspace{-0.5cm}
\end{figure}

\yh{\begin{table*}[htbp]
\setlength{\tabcolsep}{3.5pt}
\renewcommand\arraystretch{1}
\centering
\caption{\yh{The full list of all 31 drugs used in our study.}}\label{tab_alldrugs}
\yh{
\footnotesize{
\begin{tabular}{@{\extracolsep{\fill}}|cccccccc|}
\hline
ID: 1 & ID: 2 & ID: 3 & ID: 4 & ID: 5 & ID: 6 & ID: 7 & ID: 8 \\
\emph{Rasagiline} & \emph{Selegiline} & \emph{Zonisamide} & \emph{Safinamide} & \emph{Pramipexole IR/ER} & \emph{Rotigotine} & \emph{Piribedil} & \emph{Ropinirole IR/PR} \\
\hline\hline
ID: 9 & ID: 10 & ID: 11 & ID: 12 & ID: 13 & ID: 14 & ID: 15 & ID: 16  \\
\emph{Apomorphine} & \emph{Cabergoline} & \emph{Pergolide} & \emph{Entacapone} & \emph{Tolcapone} & \emph{Opicapone} & \emph{Carbidopa and Levodopa} & \emph{Istradefylline} \\
\hline\hline
ID: 17 & ID: 18 & ID: 19 & ID: 20 & ID: 21 & ID: 22 & ID: 23 & ID: 24 \\
\emph{Amantadine} & \emph{Artane} & \emph{Stilnox} &\emph{Vitamin E} & \emph{Zopiclone} & \emph{Coenzyme Q10} & \emph{Alprazolam} & \emph{Bromocriptine} \\
\hline\hline
ID: 25 & ID: 26 & ID: 27 & ID: 28 & ID: 29 & ID: 30 & ID: 31  &    \\
\emph{Memantine} & \emph{Clonazepam} & \emph{Donepezil} & \emph{Mirtazapine} & \emph{Quetiapine} & \emph{Zoloft} & \emph{Baclofen} & \\
\hline
\end{tabular}
}
}
\vspace{-0.3cm}
\end{table*}}

\section{Data Collection and Representation}
\label{sec:processing}

\subsection{Data Collection}
\yhk{We spent two years to build up a dataset of 136 PD inpatients in Nanjing Brain Hospital (NBH), China, under Ethical Approval 2015-KY-041 from the Ethics Committee of NBH, in accordance with the Principles of the Declaration of Helsinki. All of the 136 inpatients were fully informed about the research, and agreements were signed prior to the initiation of the research. Figure \ref{fig_age} summarizes the demographic information and the severity of PD symptoms of the 136 inpatients. The diagnosis and treatment of the inpatients observed the UK PD Society Brain Bank Clinical Diagnostic Criteria. In the following, we will first describe a general diagnosis and treatment procedure for any individual inpatient. Then, we will detail the medical scales, symptoms, and drugs considered in this research.}

\yhk{\textbf{General procedure.} Upon the admission of an inpatient, the symptoms are evaluated using a medical scale. The evaluation results are recorded as the symptoms matrix $\mathbf{x}$. The neurologists shall then start an initial drug therapy for the inpatient. Within the next 1 to 2 weeks, the therapy might be adjusted based on the clinical observations and feedback from the inpatient. After that, the inpatient is kept under observation with the adjusted therapy for 1 month. The adjusted therapy becomes a drug label $\mathbf{y}$ and the symptom-drug pair ($\mathbf{x}$,$\mathbf{y}$) is added to the dataset \emph{if and only if} the symptoms $\mathbf{x}$ are eventually alleviated using the drugs $\mathbf{y}$.}

\yhk{\textbf{Symptoms and the medical scale.} We adopt multiple widely used medical scales for thorough evaluations of the symptoms, including Unified PD Rating Scale (UPDRS), Non-Motor Symptoms Questionnaire (NMS-Quest), The PD Sleep Scale (PDSS), Hamilton Depression Scale (HAMD), Parkinson's Disease Questionnaire (PDQ), and Hamilton Anxiety Rating Scale (HAMA). Specifically, we consider motor symptoms such as tremor, rigidity, hand movement, leg agility, posture, gait, \etc \, as well as non-motor symptoms such as cognition, depression, anxiety, sleep/nocturnal problems, \etc \, Table \ref{tab_medicalscale} shows a full list of medical scales we used and their targets of evaluation. Table \ref{tab_medicalscale} also shows the composition of the 55 motor features and 143 non-motor features used in our dataset.}

\yhk{\textbf{Drugs.} In this research, we consider 31 drugs according to both the PD handbooks \cite{AD4}\cite{AD5} and the clinical practice in NBH, China, including including \emph{Cabergoline}, \emph{Stilnox}, \emph{Alprazolam}, \emph{Carbidopa and Levodopa}, \etc, which are shown effective in improving PD by previous studies in the field. A full list of the 31 drugs can be found in Table \ref{tab_alldrugs}. Note that some of the drugs might not be conventionally recognized as ``\emph{PD drugs}'' because they are for non-motor symptoms such as depression, anxiety, \etc}

\textbf{Motor symptoms}. The motor symptoms normally reflect whether a patient suffers from motor skills decline in daily activities (\yhk{see Table \ref{tab_medicalscale}}).
Based on the clinical observation, medical measurement and necessary quality assurance, the \yh{neurologists} graded the symptoms according to a predefined numerical scale. Some of the typical motor symptoms and their corresponding numerical grading scores are described as follows:
\begin{enumerate}
\item \textbf{\emph{Finger tapping}} $\rightarrow$ \textbf{0} (\emph{no problem}), \textbf{1} (\emph{1-2 interruptions}), \textbf{2} (\emph{3-5 interruptions}), \textbf{3} (\emph{more than 5 interruptions}), \textbf{4} (\emph{cannot perform}).
\item \textbf{\emph{Arising from chairs}} $\rightarrow$ \textbf{0} (\emph{no problem}), \textbf{1} (\emph{slower than normal}), \textbf{2} (\emph{pushes self up}), \textbf{3} (\emph{needs to push off}), \textbf{4} (\emph{cannot perform}).
\item \textbf{\emph{Swallow}} $\rightarrow$ \textbf{0} (\emph{normal}), \textbf{1} (\emph{few bucking}), \textbf{2} (\emph{occasional bucking}), \textbf{3} (\emph{only soft-diet}), \textbf{4} (\emph{nasal feed}).
\item \textbf{\emph{Dystonia}} $\rightarrow$ \textbf{0} (\emph{no}), \textbf{1} (\emph{yes}).
\item \textbf{\emph{Tetany}} $\rightarrow$ \textbf{0} (\emph{no}), \textbf{1} (\emph{slightly}), \textbf{2} (\emph{moderate}), \textbf{3} (\emph{heavy}), \textbf{4} (\emph{severe}).
\item \textbf{\emph{Using tableware}} $\rightarrow$ \textbf{0} (\emph{normal}), \textbf{1} (\emph{slow but without any help}), \textbf{2} (\emph{awkward and need help}), \textbf{3} (\emph{need help}), \textbf{4} (\emph{need feeding}).
\end{enumerate}

\textbf{Non-motor symptoms}. Different from the motor symptoms, the non-motor symptoms do not describe movement, coordination, physical tasks or mobility, but record the information such as cognitive ability, mental state, and physical condition of a subject (\yhk{see Table \ref{tab_medicalscale}}).
Some of the typical non-motor symptoms and their corresponding numerical grading scores are listed as below:
\begin{enumerate}
\item \textbf{\emph{Computing 5 simple equations, e.g., 3+7=?}} $\rightarrow$ \textbf{0} to \textbf{5} (\emph{according to the number of correct answers}).
\item \textbf{\emph{Daytime sleepiness}} $\rightarrow$ \textbf{0} (\emph{no daytime sleepiness}), \textbf{1} (\emph{can resist and stay awake}), \textbf{2} (\emph{fall asleep when alone and relaxing}), \textbf{3} (\emph{sometimes fall asleep}), \textbf{4} (\emph{often fall asleep}).
\item \textbf{\emph{Leg swelling}} $\rightarrow$ \textbf{0} (\emph{no}), \textbf{1} (\emph{yes}).
\item \textbf{\emph{Quality of sleeping at night}} $\rightarrow$ \textbf{0} to \textbf{10} (\emph{according to the quality assessment}).
\item \textbf{\emph{Insomnia}} $\rightarrow$ \textbf{0} (\emph{no}), \textbf{1} (\emph{slightly}), \textbf{2} (\emph{moderate}), \textbf{3} (\emph{heavy}), \textbf{4} (\emph{severe}).
\item \textbf{\emph{Self-depreciation}} $\rightarrow$ \textbf{0} (\emph{no}), \textbf{1} (\emph{tell when asked}), \textbf{2} (\emph{automatic tell}), \textbf{3} (\emph{initiative tell}), \textbf{4} (\emph{ecphronia}).
\end{enumerate}

The recorded numerical values are then regarded as the feature-based representations of the observed motor and non-motor symptoms, respectively.

\yhk{
\begin{table}[htbp]
\setlength{\tabcolsep}{3.5pt}
\renewcommand\arraystretch{1.1}
\centering
\caption{\yhk{Data features from multiple medical scales.}}\label{tab_medicalscale}
\footnotesize{
\begin{tabular}{@{\extracolsep{\fill}}|ccc|}
\hline
Medical Scale & Target of Evaluation & Data Features \\
\hline\hline
UPDRS (Part I) & \scriptsize{Mentation, behavior and mood} & 4 non-motor  \\
UPDRS (Part II)	& \scriptsize{Activities of daily living} &	9 motor, 4 non-motor \\
UPDRS (Part III) &	Motor disability & 27 motor \\
UPDRS (Part IV)	& Motor complications	& 11 motor \\
\scriptsize{Hoehn and Yahr Score} &	Disease severity	& 1 motor \\
NMS-Quest &	\scriptsize{Various non-motor symptoms}	& 30 non-motor \\
MMSE	& Cognition	& 20 non-motor \\
HAMD	& Depression	& 24 non-motor \\
HAMA	& Anxiety	& 14 non-motor \\
PDQ-39	& Quality of life &	7 motor, 32 non-motor \\
PDSS	& Sleep/nocturnal problems	& 15 non-motor \\
\hline
\end{tabular}
}
\end{table}
}

\subsection{Data Representation}

\yhk{From the above data collection procedure, we obtained a dataset of 136 PD inpatients. Each inpatient is represented as a 198-dimensional symptom vector $\mathbf{x}$, and is associated with a 31-dimensional label $\mathbf{y}$.}

To formally record this information, we use a binary matrix $\mathbf{Y}$ as the drug label matrix to denote the relationship between the patients and the drugs. In particular, $\mathbf{Y}_{ij} = 1$ if and only if the $i$-th patient has taken the $j$-th drug, and $\mathbf{Y}_{ij} = 0$ otherwise. Let $\mathbf{X}_1 \in \mathbb{R}^{(n+m)\times d_1}$ and $\mathbf{X}_2 \in \mathbb{R}^{(n+m)\times d_2}$ denote the feature-based representations of the observed motor and non-motor symptoms, respectively, where $n$ and $m$ denote the number of training and testing samples (\ie, subjects), respectively, and $d_1$ and $d_2$ denote the numbers of motor and non-motor symptoms, respectively. According to the data acquisition procedure discussed previously, $d_1 = 55$ and $d_2 = 143$.

Besides, we also devise several similarity-based representations to reflect patient-wise similarity. Particularly, we utilize various kernels, such as linear, Gaussian, Bhattacharyya, and $\chi$-square kernels to compute the similarity between any two row feature vectors in $\mathbf{X} \in \mathbb{R}^{(n+m)\times (d_1+d_2)}$, \ie, a concatenation of feature matrices
$\mathbf{X}_1$ and $\mathbf{X}_2$ (\ie, $\mathbf{X} = [\mathbf{X}_1, \mathbf{X}_2]$). The resulting similarity-based representations are then denoted as $\mathbf{X}_i \in \mathbb{R}^{(n+m)\times (n+m)}$ ($\forall i \geq 3$).

Both of the feature- and similarity-based representations are then combined as the multi-modality representation \cite{Lian2015Integrating}\cite{Shi13}\cite{Shi2014Joint}\cite{ADL5} of the observed motor and non-motor symptoms.
In general, the multi-modality representation of the observed symptoms for the PD patients is denoted as $\mathbf{X}_i \in \mathbb{R}^{(n+m)\times d_i}$ ($1\leq i \leq s$), where $d_i$ is the feature dimensionality of the $i$-th modality, and $s$ is the total number of different modalities. In our work, we set $s$ as 6 (\ie, 2 feature-based and 4 similarity-based representations).

\section{PALAS Model}
\label{sec:PALASmodel}

\begin{figure}[htbp]
\centering
\includegraphics[width=2in]{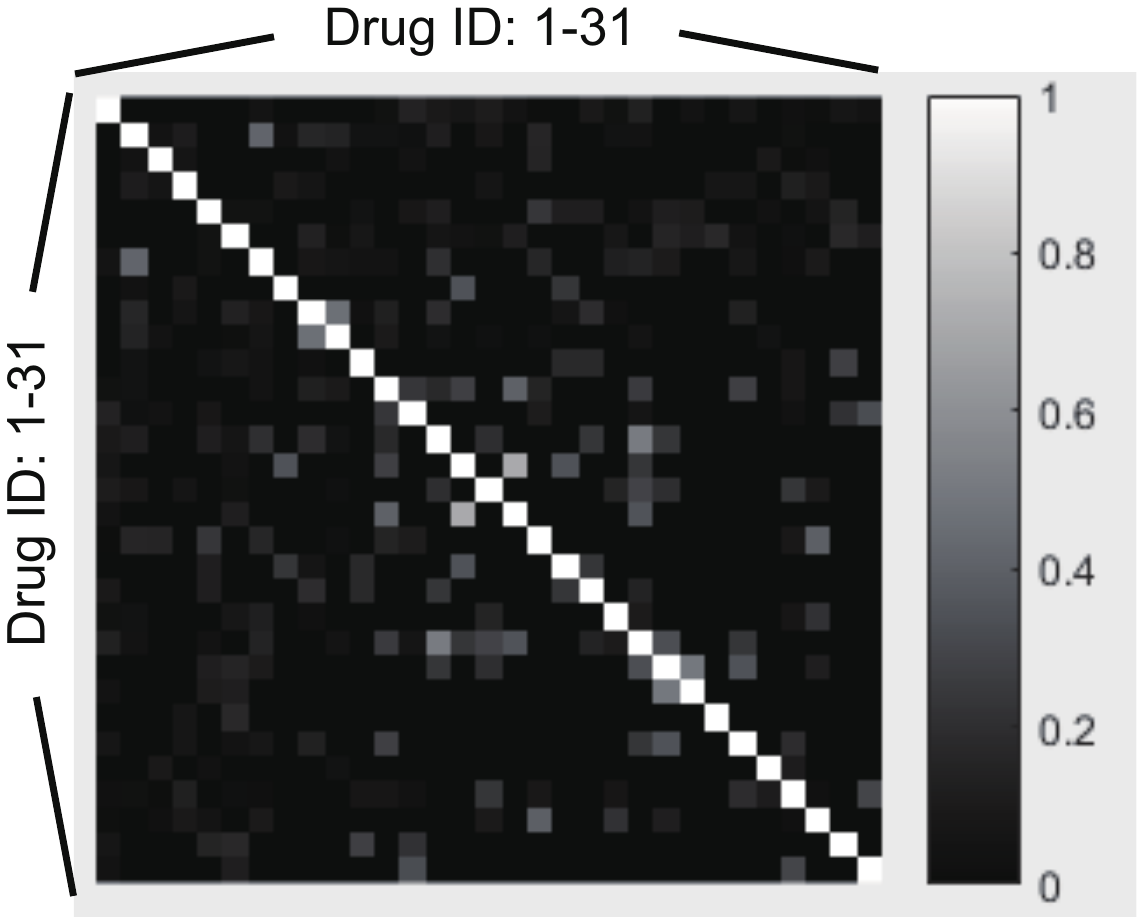}
\caption{\yh{Pearson correlation coefficient of different drugs.}}\label{exp_coefcorr}
\vspace{-0.6cm}
\end{figure}

\subsection{Problem Analysis}
To illustrate the complex relationship between the observed symptoms and treatment drugs,
we conduct an experiment by treating our prescription recommendation (prediction) task as several independent binary classification problems, and report the results in Figure \ref{exp5}.
Specifically, for each prescription drug (\ie, a column vector in the drug label matrix $\mathbf{Y}$), we train a linear SVM using the observed symptoms (\ie, the feature matrix $\mathbf{X}$ as described in Section \ref{sec:processing}) and report the average 10-fold \yh{cross-validation} classification accuracy in Figure \ref{exp5}.
\begin{figure}[htbp]
\centering
\includegraphics[width=2.6in]{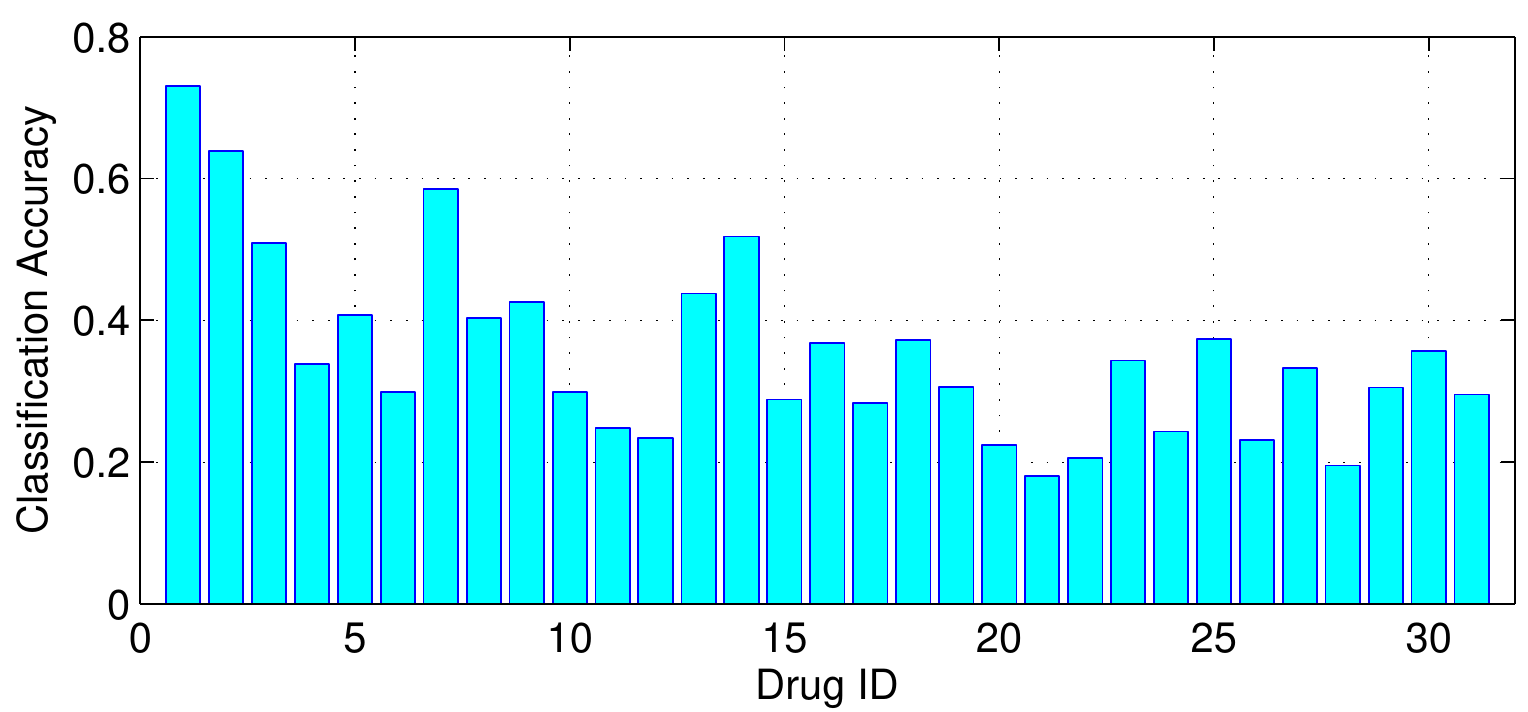}
\caption{Prediction performance of single prescription drug using only the observed symptoms $\mathbf{X}$.}\label{exp5}
\end{figure}

As shown in Figure \ref{exp5}, the classification accuracy is very low if we use only the numerical values of the observed symptoms to recommend/predict the prescription of a typical drug. Moreover, such single task prediction could not utilize the interrelation between different prescription drugs. \yh{To further understand the interrelation, the Pearson correlation coefficient is reported in Figure \ref{exp_coefcorr}. We can observe the interrelation indeed exists between different drugs.}


Thus, it is more feasible and reasonable to learn all the prediction tasks jointly (\ie, multi-label learning), and introduce a latent symptom space to exploit the intrinsic symptom-to-drug connection, so that the overall prediction performance could be improved. Specifically, by learning a latent symptom space, we in fact jointly learn two \yh{transformations}, \ie, (1) the observed symptoms to latent symptoms transformation, and (2) the latent symptoms to treatment drugs transformation.
The major considerations on learning these two transformations are given as follows:
\begin{itemize}
\item \textbf{Observed symptoms to latent symptoms transformation}. The latent symptom space is required to fully incorporate information from different modalities (\ie, $\mathbf{X}_i$) and should be able to well represent all the modalities in multi-modality representation. In other words, the learned latent symptom representation should have relatively small $F$-norm distance with the transformation of each individual modality.
\item \textbf{Latent symptoms to treatment drugs transformation}. The transformation matrix that maps the latent symptom representation to the drug labels should be sparse, as each prescription drug can only treat a limited number of symptoms.
\end{itemize}

\subsection{Problem Formulation}
Given the multi-modality representation $\mathbf{X}_i \in \mathbb{R}^{(n+m)\times d_i}$ ($i=1,...,6$ for 6 modalities) of all the samples, and the ground truth drug label matrix $\mathbf{\bar{Y}} \in \mathbb{B}^{n \times c}$ of the $n$ training samples ($\mathbb{B} = \{0,1\}$), where $c$ is the total number of drugs ($c=31$ in our work), our aim is to recommend/predict the drug labels for the $m$ test samples.
For dimension consistency, we introduce $\mathbf{Y} \in \mathbb{B}^{(n+m) \times c} = [ \mathbf{\bar{Y}}; \mathbf{0} ]$ as the (expanded) drug label matrix, in which the last $m$ rows of $\mathbf{Y}$ are set to zeros since the labels of the testing samples are not available during the training stage 

We accomplish this prediction task by proposing PALAS model, in which a latent symptom representation $\mathbf{P}\in \mathbb{R}^{(n+m)\times k}$ is learnt to connect the observed symptoms with the prescription drugs, where $k$ indicates the number of latent symptoms. The main objective function of PALAS is mathematically defined as follows:
\begin{equation}
\begin{split}
\argmin_{\mathbf{U}_i,\mathbf{P},\mathbf{V}} \quad \mathcal{F}(& \mathbf{U}_i, \mathbf{P}, \mathbf{V}) = \\
&
\sum_{i=1}^s \Big(\|\mathbf{X}_i \mathbf{U}_i - \mathbf{P}\|_\text{F}^2 + \alpha\|\mathbf{U}_i\|_\text{F}^2 \Big) \\ + & \|\mathbf{J}(\mathbf{Y}-\mathbf{PV})\|_\text{F}^2  + \beta\|\mathbf{V}\|_1,
\label{eq:obj}
\end{split}
\end{equation}
where $\mathbf{U}_i \in \mathbb{R}^{d_i \times k}$ is the observed-to-latent symptom transformation matrix to ensure that the latent symptom space $\mathbf{P}$ is able to represent the $i$-th modality (\ie, $\mathbf{X}_i$). $\mathbf{V} \in \mathbb{R}^{k \times c}$ is the symptom-to-drug transformation matrix to link different latent symptoms to different drugs. $\mathbf{V}$ is imposed to be sparse, since a typical drug can only relieve a very small number of symptoms. Also, in clinical perspective, it is rare to have a patient that should take all the possible drugs. $\mathbf{J} \in \mathbb{B}^{(n+m)\times(n+m)}$ is a binary diagonal matrix denoted as
\begin{equation}
\mathbf{J} = \text{diag}\Big(\overbrace{1,...,1}^n,\overbrace{0,...,0}^m\Big),
\end{equation}
which is used to indicate the training samples in a transductive setting, so that only the training samples are taking part in the prediction loss (\ie, second term in Eq. (\ref{eq:obj})). In this way, although the testing samples do not contribute to the drug label prediction, they are incorporated to help learning $\mathbf{P}$ and $\mathbf{U}_i$, thus addressing the issue of limited samples in this study. $\alpha$ and $\beta$ are the regularization parameters to balance different terms in Eq. (\ref{eq:obj}). Note that, the transductive setting belongs to a typical semi-supervised learning where the semi-supervised setting has been recently adopted in drug prediction or association \cite{Ding19JBHI}\cite{Liu20JBHI}.

\subsection{Optimization}
We employ alternating optimization strategy \cite{Shi15TPAMI} to solve Eq. (\ref{eq:obj}) as all of its variables are jointly convex.

\textbf{Update $\mathbf{U}_i$ with $\mathbf{P}$ and $\mathbf{V}$ fixed.} We denote $\mathbf{U}_i^{t+1}$
as the transformation matrix for the $i$-th modality at the ($t$+1)-th iteration, which can be updated from $\mathbf{U}_i^t$ (at the $t$-th iteration) by utilizing the gradient descent method as follows:
\begin{equation}
\mathbf{U}_i^{t+1} \leftarrow \mathbf{U}_i^{t} - \rho\Big(\mathbf{X}_i^\top \mathbf{X}_i\mathbf{U}_i^t-\mathbf{X}_i^\top\mathbf{P} + \alpha\mathbf{U}_i^t\Big),
\label{eq:solvu}
\end{equation}
where $\rho$ is the stepsize automatically determined by the back-tracking line search \cite{Shi15TPAMI}.

\textbf{Update $\mathbf{V}$ with $\mathbf{U}_i$ and $\mathbf{P}$ fixed.} We can obtain the following function $\mathcal{H}(\mathbf{V})$ by fixing $\mathbf{U}_i$ and $\mathbf{P}$:
\begin{equation}
\mathcal{H}(\mathbf{V}) = \|\mathbf{J}(\mathbf{Y}-\mathbf{PV})\|_\text{F}^2  + \beta\|\mathbf{V}\|_1,
\label{eq:objvp}
\end{equation}
which is a combination of a smooth term and a $L_1$-norm. We employ FISTA (Fast Iterative Soft-Thresholding Algorithm) \cite{Beck:2009:FIS:1658360.1658364} for solving $\mathbf{V}$ in Eq. (\ref{eq:objvp}). We summarize the optimization steps as follows:
\begin{equation}
\begin{split}
\textbf{V}^{t'} & \leftarrow\mathcal{G}_{\epsilon}\Big(\mathbf{\Gamma}^{t'}-\frac{1}{l_{t'}}\nabla\mathcal{H}(\mathbf{\Gamma}^{t'})\Big), \\
\psi_{t'+1} & \leftarrow\frac{1+\sqrt{1+4\psi_{t'}^2}}{2}, \\
\mathbf{\Gamma}^{t'+1} & \leftarrow\textbf{V}^{t'}+\Big(\frac{\psi_{t'}-1}{\psi_{t'+1}}\Big)\Big(\textbf{V}^{t'}-\textbf{V}^{t'-1}\Big),
\end{split}
\label{eq:solvv}
\end{equation}
where $t'$ is used to indicate the inner iteration during solving $\mathbf{V}$ which is different from outer iteration $t$ in the main objective. $\mathbf{\Gamma} \in \mathbb{R}^{k\times c}$ is an intermediate variable initialized as $\mathbf{\Gamma}^1 = \mathbf{V}^0$ ($\mathbf{V}^0$ is the initialization before performing FISTA). $\psi_{t'} \in \mathbb{R}$ is a variable initialized as $\psi_1=1$. $l_{t'}$ is the stepsize automatically determined by the back-tracking line search.
Meanwhile, $\nabla\mathcal{H}(\mathbf{\Gamma}^{t'})$ is the gradient at the $t'$-th iteration, given as
\begin{equation}
\nabla\mathcal{H}(\mathbf{\Gamma}^{t'}) = -\mathbf{P}^{\top}\mathbf{J}\Big(\mathbf{Y}-\mathbf{P} \mathbf{\Gamma}^{t'}\Big).
\end{equation}

In Eq. (\ref{eq:solvv}), $\mathcal{G}_{\epsilon}$ is the shrinkage function, given as
\begin{equation}
\mathcal{G}_{\epsilon}(x) = \left\{\begin{array}{ll}
x-\epsilon & \text{if} \; x > \epsilon,\\
x+\epsilon & \text{if} \; x < \epsilon,\\
0 & \text{otherwise}
\end{array}\right.
\end{equation}

Thus, with fixed $\mathbf{U}_i$ and $\mathbf{P}$, we can eventually obtain the updated $\mathbf{V}$ by repeating the steps in Eq. (\ref{eq:solvv}) until convergence.


\textbf{Update $\mathbf{P}$ with $\mathbf{U}_i$ and $\mathbf{V}$ fixed.} To update $\mathbf{P}$, we employ gradient descent with the following updating rule:
\begin{equation}
\mathbf{P}^{t+1} \leftarrow \mathbf{P}^t - \rho\Big(\mathbf{J}(\mathbf{Y}-\mathbf{P}^t\mathbf{V})\mathbf{V}^\top + \sum_{i=1}^s(\mathbf{X}_i\mathbf{U}_i-\mathbf{P}^t)\Big).
\label{eq:solvp}
\end{equation}

Finally, we summarized all the optimization steps for PALAS model in Algorithm \ref{alg:PALAS}. For initialization, $\mathbf{U}_i$ and $\mathbf{V}$ are set to
\begin{equation}
\mathbf{U}_i = \frac{\mathbf{1}}{d_i k},
\end{equation}
and
\begin{equation}
\mathbf{V} = \frac{\mathbf{1}}{k c},
\end{equation}
where $\mathbf{1}$ is the all-one matrix/vector. In addition, we perform unsupervised PCA on $\mathbf{X}$ and reduce the original dimension (\ie, $d_1+d_2$) to $k$, as the initialization for $\mathbf{P}$. We then repeat Eq. (\ref{eq:solvu}), Eq. (\ref{eq:solvv}), and Eq. (\ref{eq:solvp}) to update $\mathbf{U}_i$, $\mathbf{V}$, and $\mathbf{P}$, respectively, until convergence.

\textbf{Test Sample Prediction}: Given a testing sample $\mathbf{z}$ represented by $\{\mathbf{z}_i \in \mathbb{R}^{d_i}\}$ ($1\leq i \leq s$), the predicted drug labels $\mathbf{\hat{y}} \in \mathbb{R}^{c}$ can be obtained by
\begin{equation}
\mathbf{\hat{y}} = \text{sgn}\Big(\frac{1}{s}\sum_{i=1}^s \mathbf{z}_i\mathbf{U}_i\mathbf{V}\ - \frac{\mathbf{1}}{2} \Big),
\end{equation}
where $\mathbf{1}$ is all-one vector. 

\subsection{Convergence Analysis}

\begin{lemma}
\cite{Beck:2009:FIS:1658360.1658364}
Let $\mathbf{V}^*$ be the optimal solution of Eq. (\ref{eq:objvp}) and $\eta$ be the Lipschitz constant, $\forall t\geq 1$,
\begin{displaymath}
\mathcal{H}(\mathbf{V}^t)-\mathcal{H}(\mathbf{V}^*) \leq \frac{2l_t \eta\|\mathbf{V}^0-\mathbf{V}^*\|^2}{(1+t)^2}.
\end{displaymath}
\end{lemma}

\begin{theorem}
The optimization process (\ie, Algorithm \ref{alg:PALAS}) monotonically decreases the objective of Eq. (\ref{eq:obj}) in each iteration and will converge to a local optimal solution.
\end{theorem}

\begin{proof}
By performing Algorithm \ref{alg:PALAS}, we obtain the updated $\mathbf{U}_i^{t+1}$, $\mathbf{P}^{t+1}$, $\mathbf{V}^{t+1}$ from $\mathbf{U}_i^{t}$, $\mathbf{P}^{t}$, $\mathbf{V}^{t}$ in each iteration, respectively. Thus, performing Eq. (\ref{eq:solvu}) and Eq. (\ref{eq:solvp}) and incorporating \textbf{Lemma 1}, we get
\begin{equation}
\begin{split}
\forall t\geq 1, \mathcal{F} (\mathbf{U}_i^t,\mathbf{P}^t,\mathbf{V}^t) \quad \overbrace{\geq}^{\text{Eq. (\ref{eq:solvu})}} \quad &  \mathcal{F}(\mathbf{U}_i^{t+1},\mathbf{P}^t,\mathbf{V}^t) \\
\quad \overbrace{\geq}^{\text{Lemma 1}} \quad & \mathcal{F}(\mathbf{U}_i^{t+1},\mathbf{P}^t,\mathbf{V}^{t+1}) \\
\quad \overbrace{\geq}^{\text{Eq. (\ref{eq:solvp})}} \quad & \mathcal{F}(\mathbf{U}_i^{t+1},\mathbf{P}^{t+1},\mathbf{V}^{t+1}).
\end{split}
\end{equation}
Therefore, the convergence could be satisfied. $\hfill\ensuremath{\blacksquare}$
\end{proof}

      \begin{algorithm}[H]
        \caption{Optimization of PALAS}\label{alg:PALAS}
        \renewcommand{\algorithmicrequire}{\textbf{Input:}}
        \renewcommand{\algorithmicensure}{\textbf{Output:}}
        \begin{algorithmic}
        \REQUIRE $\forall i\,\mathbf{X}_i$, $\mathbf{Y}$
        \STATE \textbf{Initialization}:
        \STATE $\forall i\,\mathbf{U}_i = \frac{\mathbf{1}}{d_i k}$, $\mathbf{V} = \frac{\mathbf{1}}{k c}$
        \STATE $\mathbf{P} \leftarrow$ Performing PCA on $\mathbf{X}$
        \WHILE{not converged}
        \STATE$\forall i \,\mathbf{U}_i \leftarrow$ Update by Eq. (\ref{eq:solvu}).
        \STATE$\mathbf{V} \leftarrow$ Update by Eq. (\ref{eq:solvv}).
        \STATE$\mathbf{P} \leftarrow$ Update by Eq. (\ref{eq:solvp}).
        \ENDWHILE
        \ENSURE $\forall i\,\mathbf{U}_i$, $\mathbf{V}$, $\mathbf{P}$
        \end{algorithmic}
      \end{algorithm}

\subsection{Complexity}
To analyze the complexity of our optimization algorithm, for each outer iteration, we denote $T_1$ as the number of inner iteration for solving $\mathbf{V}$ by repeating Eq. (\ref{eq:solvv}). Also, we denote $\hat{d} = \sum_{i=1}^s d_i$, $\bar{d} = \sum_{i=1}^s d_i^2$.
Thus, for solving all $\mathbf{U}_i$ ($1\leq i \leq s$), the complexity is $\mathcal{O}(\bar{d}k+ \bar{d}n + \hat{d}nk)$; For solving $\mathbf{V}$, the complexity is $\mathcal{O}(n^2kT_1 + nkcT_1)$; For solving $\mathbf{P}$, the complexity is $\mathcal{O}(n^2 c + nkc + \hat{d}nk)$. We denote the number of outer iterations as $T$. In summary, as FISTA is fast to converge (\ie, $T_1\ll T$), the major complexity of our optimization is $\mathcal{O}(n^2cT + \hat{d}nkT + n^2kT_1T + nkcT_1T)$.

\section{Results}
\label{sec:experiments}
\begin{figure*}[htbp]
\centering
\includegraphics[width=1.6in]{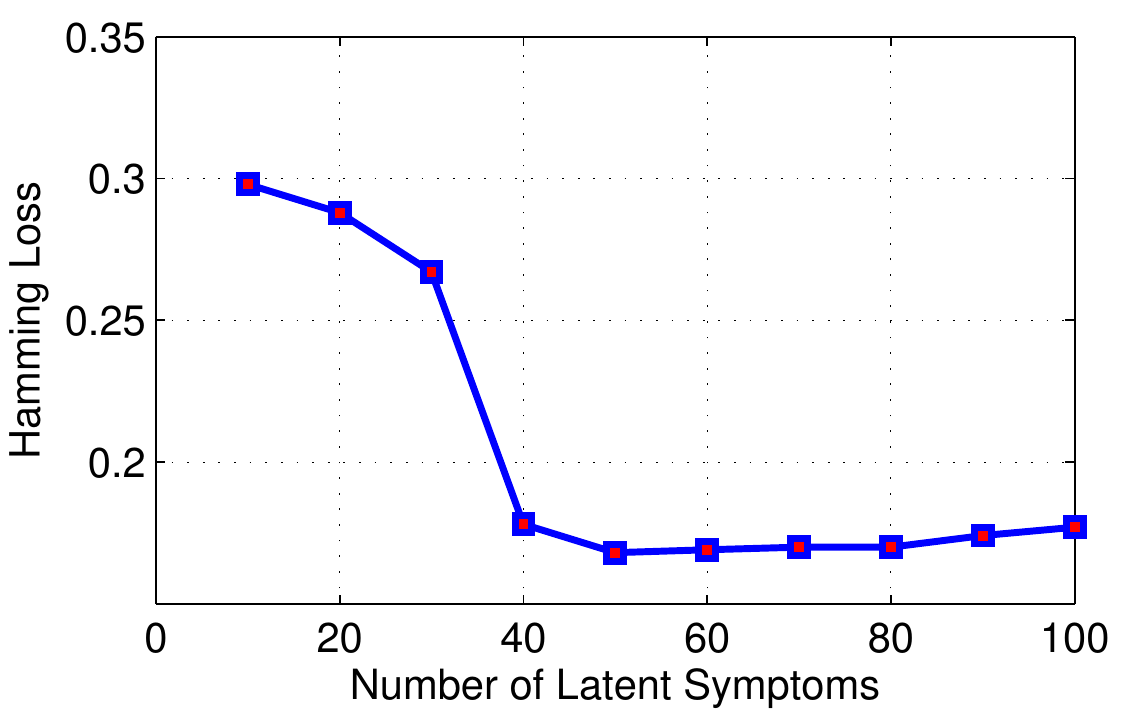}
\includegraphics[width=1.6in]{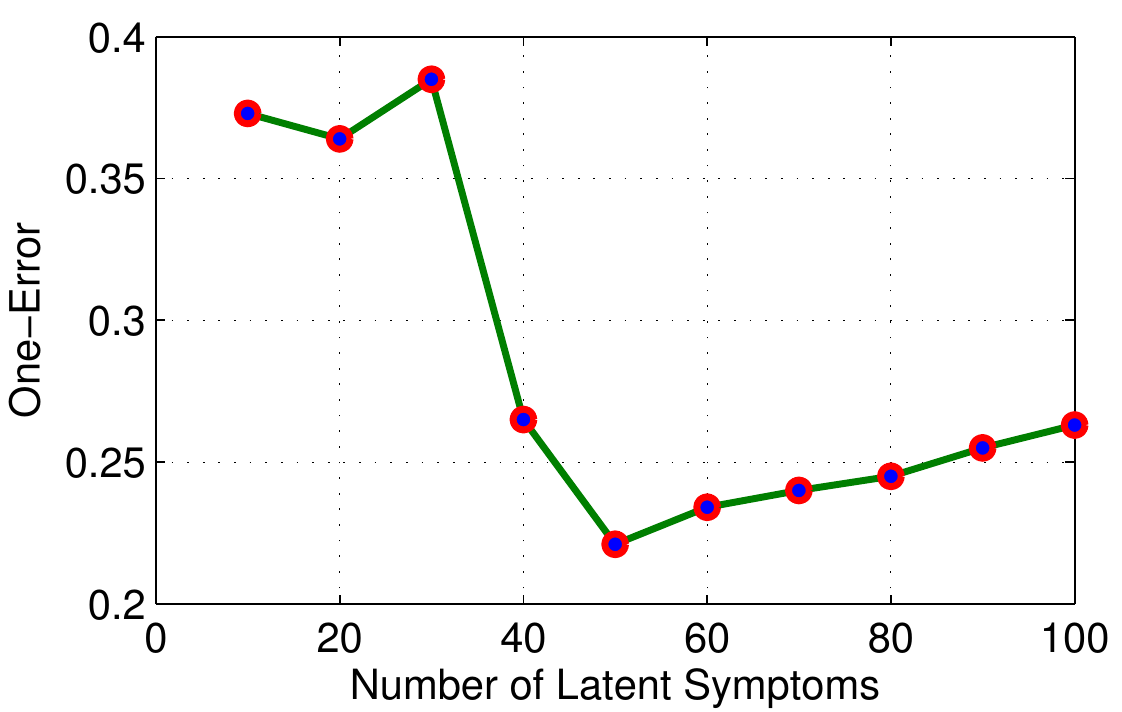}
\includegraphics[width=1.6in]{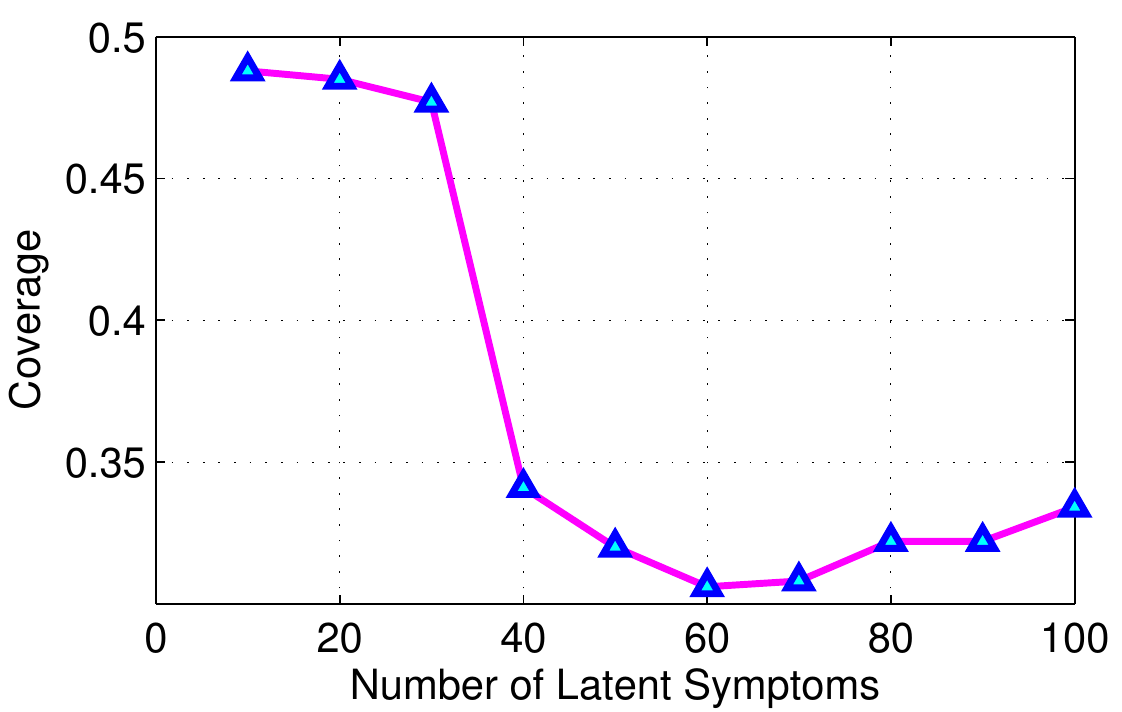}
\includegraphics[width=1.6in]{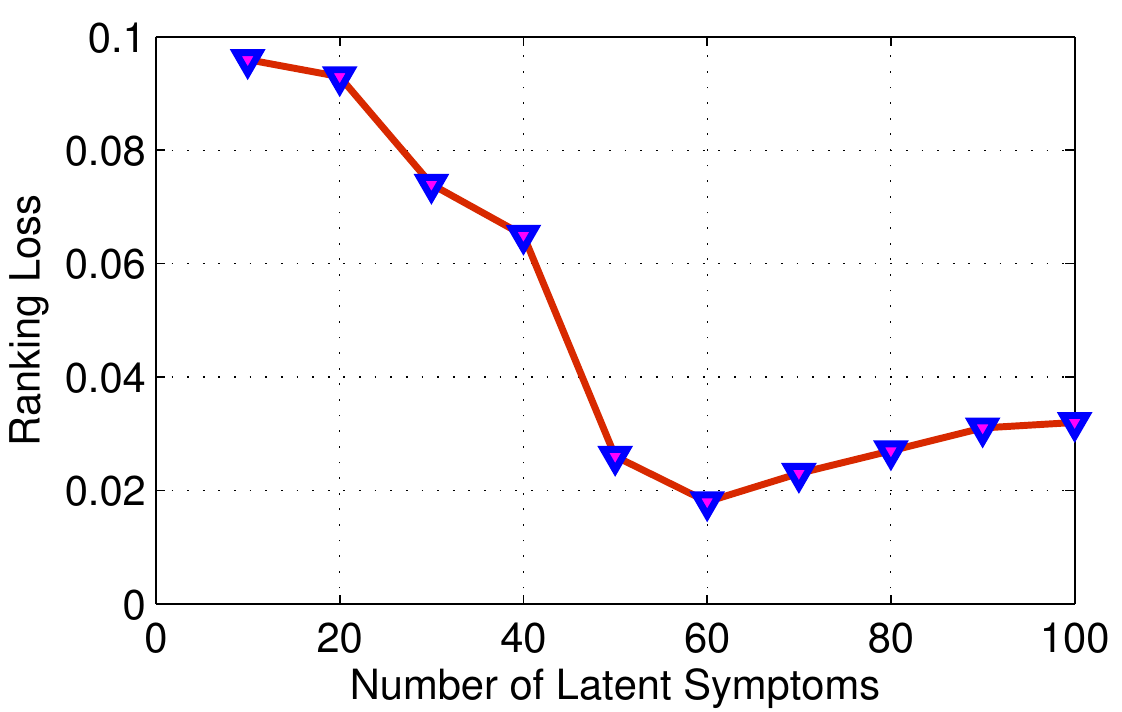}
\caption{The prediction performance of the PALAS model using different $k$ values.}\label{exp3}
\end{figure*}

\subsection{Experimental Setting}
\label{sec:experiments:A}
\textbf{Evaluation Metrics.}
As our goal is to recommend/predict the best subsets from all the treatment drugs, \yh{this belongs to a conventional multi-label learning setting to predict the labels of 31 drugs.} Inspired by multi-label learning \cite{Zhang:2007:MLL:1234417.1234635}\cite{Elisseeff01akernel}\cite{Clare2002Knowledge}, we employ the \texttt{Hamming loss}, \texttt{one-error}, \texttt{coverage}, and \texttt{ranking loss} as our evaluation metrics. All of them are the lower the better.
\yh{Moreover, we report the mean accuracy, mean sensitivity and mean specificity on different drugs for evaluation. All of them are the higher the better.}
Beside these seven metrics, we also use clinical metrics provided by three \yh{neurologists}, which we will clarify shortly.

We implemented our method in a PC with Intel i7-core 3.7GHz and 32GB RAM. In addition, we also performed 100 times of 10-fold \yh{cross-validation} (\ie, 14 samples in the 1-st to 9-th folds, and 10 samples in the 10-th fold), and averaged the obtained values as the final results, to avoid possible data-split bias.

\textbf{Parameter Setting.}
For choosing the appropriate parameters, we randomly choose the 90$\%$ data as training set to investigate the different combinations of parameters according to their performance on the rest 10$\%$ data as testing set, which aims to automatically determine the values of three parameters in PALAS, \ie, $\alpha$, $\beta$, and the number of latent symptoms $k$. In particular, $\alpha$ and $\beta$ are chosen from 15 different values (\ie, $1$, $0.5$, $0.3$, $0.1$, $0.05$, $0.01$, $0.005$, $0.001$, $5\times10^{-4}$, $10^{-4}$, $5\times10^{-5}$, $10^{-5}$, $10^{-6}$, $10^{-8}$, $10^{-10}$), while $k$ is chosen from $10$, $20$, ..., $100$. The combination of $\alpha$, $\beta$ and $k$ that achieves lowest value in \texttt{Hamming loss} on this chosen testing set is selected as the final parameters in PALAS.

According to this parameter choosing process, the combination of $\alpha=0.3$, $\beta=0.1$, and $k=50$ is found to be best. Thus, all the following results are reported based on this parameter setting.


\begin{table*}[htbp]
\setlength{\tabcolsep}{3.5pt}
\renewcommand\arraystretch{1.1}
\centering
\caption{Performance of the prescription drug prediction between our proposed method (\ie, PALAS) and the related baseline methods. \yh{$\downarrow$ indicates the lower the better, and $\uparrow$ indicates the higher the better.}}\label{exp1}
\footnotesize{
\begin{tabular}{@{\extracolsep{\fill}}|c|cccccc|}
\hline
Method & FC & BR & CLR & ML-$k$NN & Rank-SVM & SMBR \\
\hline\hline
\texttt{Hamming loss}\scriptsize{$\downarrow$} & 0.574 $\pm$ 0.032 & 0.210 $\pm$ 0.024 & 0.195 $\pm$ 0.021 & 0.175 $\pm$ 0.016 & 0.184 $\pm$ 0.026 & 0.241 $\pm$ 0.015 \\
\texttt{One-error}\scriptsize{$\downarrow$} & 0.536 $\pm$ 0.140 & 0.283 $\pm$ 0.127 & 0.264 $\pm$ 0.115 & 0.250 $\pm$ 0.107 & 0.260 $\pm$ 0.125 & 0.292 $\pm$ 0.095 \\
\texttt{Coverage}\scriptsize{$\downarrow$} & 0.655 $\pm$ 0.095 & 0.406 $\pm$ 0.078 & 0.417 $\pm$ 0.063 & 0.364 $\pm$ 0.070 & 0.365 $\pm$ 0.067 & 0.336 $\pm$ 0.051 \\
\texttt{Ranking loss}\scriptsize{$\downarrow$} & 0.102 $\pm$ 0.007 & 0.033 $\pm$ 0.003 & 0.033 $\pm$ 0.001 & \textbf{0.023 $\pm$ 0.003} & 0.028 $\pm$ 0.006 & 0.046 $\pm$ 0.008 \\
\yh{\texttt{Sensitivity} \scriptsize{$\uparrow$}} & \yh{0.657 $\pm$ 0.004} & \yh{0.794 $\pm$ 0.002} & \yh{0.785 $\pm$ 0.002} & \yh{0.823 $\pm$ 0.002} & \yh{0.817 $\pm$ 0.002} & \yh{0.774 $\pm$ 0.003} \\
\yh{\texttt{Specificity} \scriptsize{$\uparrow$}} & \yh{0.934 $\pm$ 0.002} & \yh{0.970 $\pm$ 0.002} & \yh{0.998 $\pm$ 0.001} & \yh{0.996 $\pm$ 0.001} & \yh{0.965 $\pm$ 0.001} & \yh{0.998 $\pm$ 0.001} \\
\yh{\texttt{Accuracy} \scriptsize{$\uparrow$}} & \yh{0.884 $\pm$ 0.005} & \yh{0.943 $\pm$ 0.003} & \yh{0.947 $\pm$ 0.002} & \yh{0.969 $\pm$ 0.003} & \yh{0.953 $\pm$ 0.003} & \yh{0.942 $\pm$ 0.003} \\
\hline
\hline
Method & MLMVL-MM & MVML & PALAS ($k$=50) & PALAS$_\texttt{wt}$ ($k$=50) & PALAS ($k$=60) & PALAS$_\texttt{tv}$ ($k$=50) \\
\hline\hline
\texttt{Hamming loss}\scriptsize{$\downarrow$} & 0.173 $\pm$ 0.012 & 0.195 $\pm$ 0.028 & \textbf{0.168 $\pm$ 0.017} & 0.169 $\pm$ 0.020 & 0.169 $\pm$ 0.017 & 0.172 $\pm$ 0.017 \\
\texttt{One-error}\scriptsize{$\downarrow$} & 0.231 $\pm$ 0.093 & 0.271 $\pm$ 0.133 & \textbf{0.221 $\pm$ 0.110} & 0.226 $\pm$ 0.121 & 0.240 $\pm$ 0.109 & 0.228 $\pm$ 0.142 \\
\texttt{Coverage}\scriptsize{$\downarrow$} & 0.325 $\pm$ 0.063 & 0.341 $\pm$ 0.081 & 0.320 $\pm$ 0.065 & 0.324 $\pm$ 0.073 &\textbf{0.308 $\pm$ 0.065} & 0.324 $\pm$ 0.088\\
\texttt{Ranking loss}\scriptsize{$\downarrow$} & 0.028 $\pm$ 0.004 & 0.031 $\pm$ 0.008 & 0.026 $\pm$ 0.006 & 0.026 $\pm$ 0.008 & \textbf{0.023 $\pm$ 0.006} & 0.026 $\pm$ 0.006  \\
\yh{\texttt{Sensitivity} \scriptsize{$\uparrow$}} & \yh{0.827 $\pm$ 0.003} & \yh{0.793 $\pm$ 0.003} & \yh{\textbf{0.841 $\pm$ 0.002}} & \yh{0.839 $\pm$ 0.002} & \yh{0.837 $\pm$ 0.002} & \yh{0.825 $\pm$ 0.002} \\
\yh{\texttt{Specificity} \scriptsize{$\uparrow$}} & \yh{0.998 $\pm$ 0.001} & \yh{0.984 $\pm$ 0.002} & \yh{\textbf{0.999 $\pm$ 0.000}} & \yh{0.998 $\pm$ 0.001} & \yh{\textbf{0.999 $\pm$ 0.001}} & \yh{0.997 $\pm$ 0.001} \\
\yh{\texttt{Accuracy} \scriptsize{$\uparrow$}} & \yh{0.976 $\pm$ 0.002} & \yh{0.945 $\pm$ 0.003} & \yh{\textbf{0.983 $\pm$ 0.002}} & \yh{0.979 $\pm$ 0.002} & \yh{0.982 $\pm$ 0.002} & \yh{0.978 $\pm$ 0.003} \\
\hline
\end{tabular}
}
\end{table*}

\begin{figure*}
\centering
\begin{subfigure}[t]{3in}
\centering
\includegraphics[angle=270,width=3in]{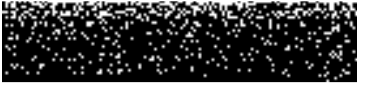}
\vspace{-0.2cm}
\caption{Prediction: $\beta$=$10^{-5}$}\label{exp4:1}
\end{subfigure}
\begin{subfigure}[t]{3in}
\centering
\includegraphics[angle=270,width=3in]{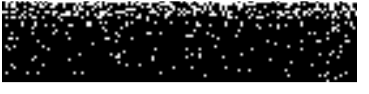}
\vspace{-0.2cm}
\caption{Prediction: $\beta$=$10^{-3}$}\label{exp4:2}
\end{subfigure}\\
\vspace{0.2cm}
\begin{subfigure}[t]{3in}
\centering
\includegraphics[angle=270,width=3in]{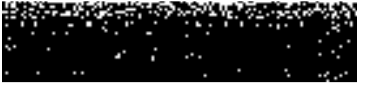}
\vspace{-0.2cm}
\caption{Prediction: $\beta$=$0.1$}\label{exp4:3}
\end{subfigure}
\begin{subfigure}[t]{3in}
\centering
\includegraphics[angle=270,width=3in]{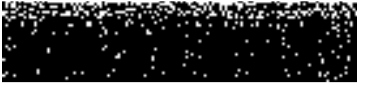}
\vspace{-0.2cm}
\caption{Ground Truth}\label{exp4:4}
\end{subfigure}%
\caption{The role of $\beta$. From top left to bottom right correspond to the predicted drug label matrices with $\beta$=$10^{-5}$, $\beta$=$10^{-3}$, and $\beta$=$0.1$, and the ground truth label matrix, respectively. Rows indicate different drugs, and columns indicate different patients. White dot means the corresponding patient takes the corresponding drug, while black dot means not. Larger $\beta$ value will resulting to sparser predicted results.}\label{exp4}
\end{figure*}

\textbf{Comparison Baselines.}
We compare our method with the following baseline algorithms that were developed to deal with multi-modality data and (or) multiple labels.
\begin{itemize}
\item Feature Concatenation (FC): treating our problem as several independent binary classification problems solved via linear SVMs trained using only the observed symptoms (\ie, as discussed in Section \ref{sec:PALASmodel}).
\item Sparse Multimodal Biometrics Recognition (SMBR) \cite{6529074}: a joint sparse representation for multi-modality data.
\item Binary Relevance (BR) \cite{BoutellCR20041757}: learning several independent binary classifiers for multi-label learning.
\item Calibrated Label Ranking (CLR) \cite{Furnkranz:2008:MCV:1416930.1416934}: treating multi-label learning task as a typical label ranking problem.
\item Rank-SVM \cite{Elisseeff01akernel}: maximizing the margin for multi-label data by incorporating both the relevant and the irrelevant labels.
\item ML-$k$NN \cite{Zhang:2007:MLL:1234417.1234635}: a multi-label version of traditional kNN by maximizing the posterior probability.
\item MLMVL-MM \cite{Zhangzhongfei2012}: learning a multi-modality consensus representation by imposing maximum margin criterion.
\item MVML \cite{Zou2016}: combining multiple single-label learners in a boosting framework, with a modified multi-modality SVM used as the base learner.
\end{itemize}


We use the publicly available implementations for SMBR\footnote{http://users.umiacs.umd.edu/~sshekha/robustmulti.html}, BR\footnote{http://mulan.sourceforge.net/doc/}, CLR\footnote{http://mulan.sourceforge.net/doc/}\cite{Tsoumakas10}, ML-$k$NN\footnote{http://cse.seu.edu.cn/PersonalPage/zhangml/files/ML-kNN.rar}, and Rank-SVM\footnote{http://cse.seu.edu.cn/PersonalPage/zhangml/files/RankSVM.rar} in this study. For MLMVL-MM and MVML, as their implementations are not publicly available, we implement them ourselves. All the parameters in these baselines are automatically determined by the same aforementioned process. For multi-label learning methods (\ie, BR, CLR, ML-$k$NN, Rank-SVM), we combine the observed motor and non-motor symptoms (two modalities) into a feature matrix ($\mathbf{X}$ in Section \ref{sec:processing}) during the learning process.

Please also note that, we did not include the similarity-based representations for these multi-label learning methods, as their results are inferior to the results that only use the feature-based representation. Furthermore, since SMBR is specifically designed for single label prediction, we perform SMBR for each drug separately, and combine all the predictions as final results.

\subsection{Comparison with Baselines}
Table \ref{exp1} shows the prediction performance of our method (PALAS) and the baseline methods.
Note that, we choose the PALAS ($k$=50) in Table \ref{exp1} as our final model which is used as the reference for all the comparison in the following evaluation (including the pair-wise $t$-test).
From the table, it can be seen that our PALAS model (\ie, PALAS ($k$=50) in Table \ref{exp1}) outperforms all the baseline methods in all the four evaluation metrics, except for the \texttt{Ranking loss} metric, where the ML-$k$NN performs as well as ours.

Besides, we also conduct the pair-wise $t$-test between the results from PALAS and the results from all the baseline methods. The statistical test shows that the PALAS model is statistically better than all the comparison methods (except ML-$k$NN, if using the \texttt{Ranking loss} evaluation metric).
Also, we introduce the setting of our method directly trained on the original two-view features (\ie, motor and non-motor), and name this setting as PALAS$_\texttt{tv}$ ($k$=50) in Table \ref{exp1}. \yh{The statistical comparison on \texttt{Hamming loss} was conducted and $p$-values were reported in Table \ref{exp_sup_pvalue}, which demonstrated that our PALAS was statistically better than these baselines.}

Please note that we also report the results of PALAS when the latent symptoms number equals to 60 for reference (\ie, PALAS ($k$=60) in Table \ref{exp1}).

In addition, we degenerate our PALAS model (\ie, PALAS ($k$=50) in Table \ref{exp1}) to its non-transductive version, namely PALAS$_\texttt{wt}$ (\ie, PALAS$_\texttt{wt}$ ($k$=50) in Table \ref{exp1}). This can be realized by removing $\mathbf{J}$ in Eq. (\ref{eq:obj}) and only using the training samples to learn $\mathbf{P}$. The comparison results reported in Table \ref{exp1} show that \yh{a slight performance improvement} can be achieved under the transductive setting.
From the table, we also note that, without the transductive setting, the standard deviation of PALAS$_\texttt{wt}$ is generally slightly higher than that of PALAS, implying that the transductive setting has more stable performance.


\yh{
\begin{table}[htbp]
\vspace{-0.3cm}
\setlength{\tabcolsep}{3.5pt}
\renewcommand\arraystretch{1}
\centering
\caption{\yh{The $p$-values of PALAS against other methods on \texttt{Hamming loss}.}}\label{exp_sup_pvalue}
\footnotesize{
\begin{tabular}{@{\extracolsep{\fill}}|c|cc|}
\hline
\footnotesize{PALAS against Method} & $p$-values & \footnotesize{Is significantly better?} \\
\hline\hline
FC & $6.21 \times 10^{-15}$ & $\checkmark$  \\
BR & $3.74 \times 10^{-13}$ & $\checkmark$  \\
CLR & $2.85 \times 10^{-4}$ & $\checkmark$  \\
ML-$k$NN & $7.73 \times 10^{-3}$ & $\checkmark$  \\
Rank-SVM & $9.61 \times 10^{-4}$ & $\checkmark$  \\
SMBR & $9.47 \times 10^{-9}$ & $\checkmark$  \\
MLMVL-MM & $7.44 \times 10^{-3}$ & $\checkmark$  \\
MVML & $3.65 \times 10^{-4}$ & $\checkmark$  \\
\hline
\end{tabular}
}
\end{table}
}

Compared with MLMVL-MM and PALAS (\ie, PALAS ($k$=50) in Table \ref{exp1}), MVML ignores the relation among different drugs as it directly models the (multiple label) learning task as a combination of several single-label learning problems. The result in Table \ref{exp1} shows that, the relation of multiple labels can be well used for performance improvement for MLMVL-MM and PALAS. Moreover, by comparing with (single-modality) multiple label learning methods, we can infer that multi-modality setting of learning the intrinsic representation across different modalities could help guide better results. \yh{Finally, the confusion matrix of all 31 drugs is provided in Figure \ref{fig_confmat}.}


\subsection{Sensitivity Analysis}
For a better understanding of our proposed PALAS model, taking PALAS ($k$=50) in Table \ref{exp1} as an example, we investigate how the parameter setting and the tranductive component affect its performance.
First, we investigate the influence of the parameter $k$, \ie, the dimension of the latent space. Figure \ref{exp3} shows the model performance by varying different number of $k$ values, from 10 to 100.
Here, $\alpha$ and $\beta$ are automatically determined by the given $k$ which follows the same parameter selection process as that in Section \ref{sec:experiments:A}: the combination of $\alpha$ and $\beta$ that achieves lowest value in \texttt{Hamming loss} on a randomly-chosen testing set (by randomly sampling 10$\%$ data from the dataset) is selected.

From Figure \ref{exp3}, it is observed that the setting $k=50$ leads to better \texttt{Hamming loss} and \texttt{One-error}, while the setting $k=60$ leads to better \texttt{Ranking loss} and \texttt{Coverage}. In addition, we can see that small value of $k$ will result to poor performance, which is understandable as that latent space cannot fully capture the relationship between the observed symptoms and the prescription drugs. On the other hand, setting the value of $k$ too large will also decline the performance, probably due to the noise and overfitting issues that start to kick in for high dimensional data.

Next, we investigate the influence of the parameter $\beta$, which is used to control the sparsity of the symptom-to-drug transformation matrix $\mathbf{V}$ in Eq. (\ref{eq:obj}). This sparsity constraint is reasonable as most patients would only take a very few from all the candidates of treatment drugs, instead of being prescribed all kinds of treatment drugs, as shown in Figure \ref{exp4}(d).
Thus, we visually investigate the final predicted drug label matrix by using different values of $\beta$, as shown in Figure \ref{exp4}. From the figure, we can observe that the larger the value of $\beta$, the sparser the predicted label matrix, in which at the $\beta = 0.1$, the sparse predicted label matrix (Figure \ref{exp4}(c)) closely resembles the ground truth label matrix (Figure \ref{exp4}(d)). Thus, imposing sparsity constraint on $\mathbf{V}$ will likely benefit the predicted results.

Furthermore, to investigate the sensitivity of the PALAS model to the weight parameters (\ie, $\alpha$ and $\beta$), we show the performance of the PALAS model using different combinations of $\alpha$ and $\beta$ in Figure \ref{exp6}. Please note that, for simplicity, the values 1 to 15 in the x and y axes actually denote the corresponding indices of the aforementioned 15 values (\ie, $1$, $0.5$, $0.3$, $0.1$, $0.05$, $0.01$, $0.005$, $0.001$, $5\times10^{-4}$, $10^{-4}$, $5\times10^{-5}$, $10^{-5}$, $10^{-6}$, $10^{-8}$, $10^{-10}$).

\begin{figure*}[htbp]
\centering
\includegraphics[width=6.4in]{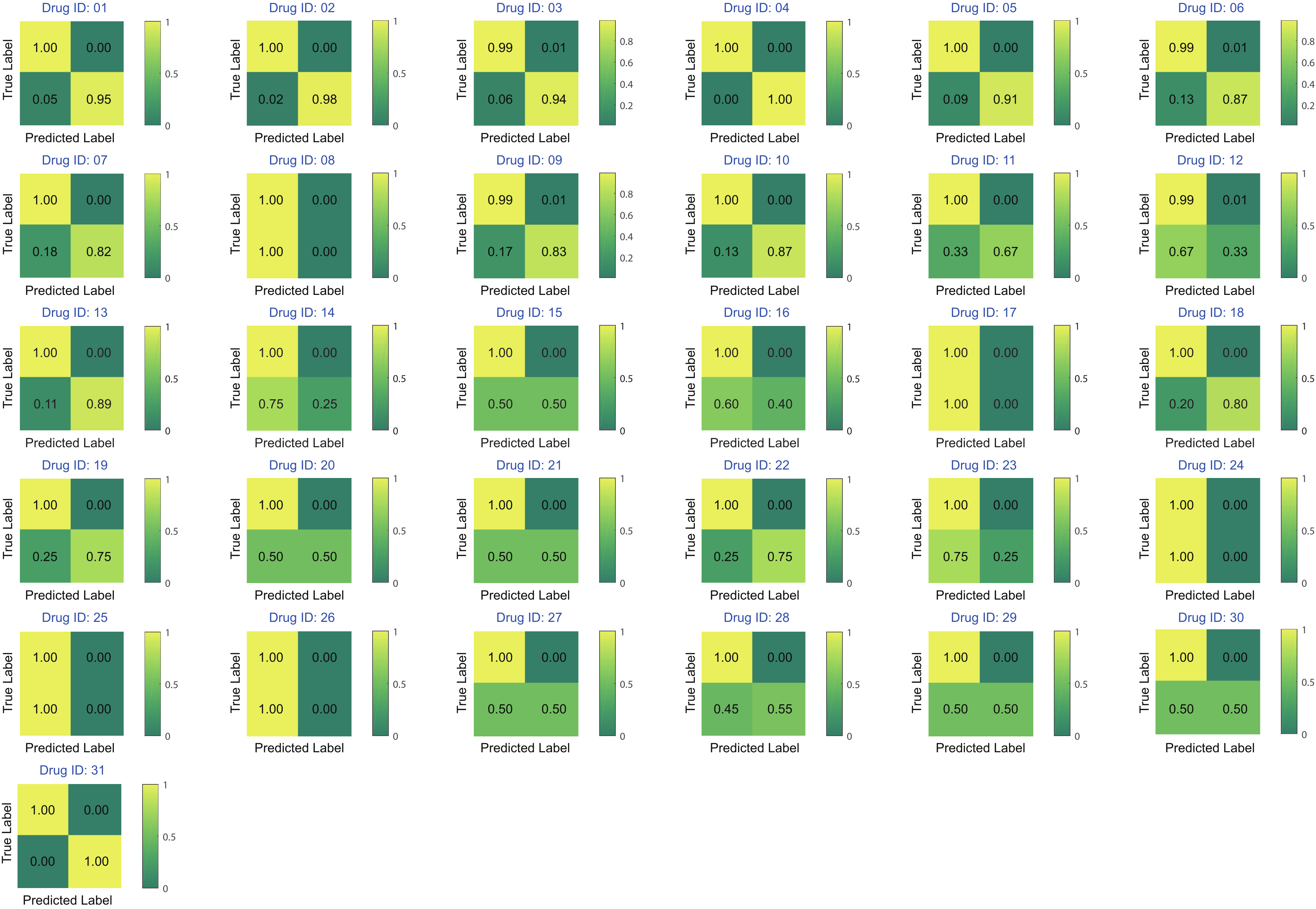}
\caption{\yh{The confusion matrix of different prescription drug.}}\label{fig_confmat}
\vspace{-0.3cm}
\end{figure*}

\begin{figure*}[htbp]
\centering
\begin{subfigure}[t]{1.7in}
\centering
\includegraphics[width=1.7in]{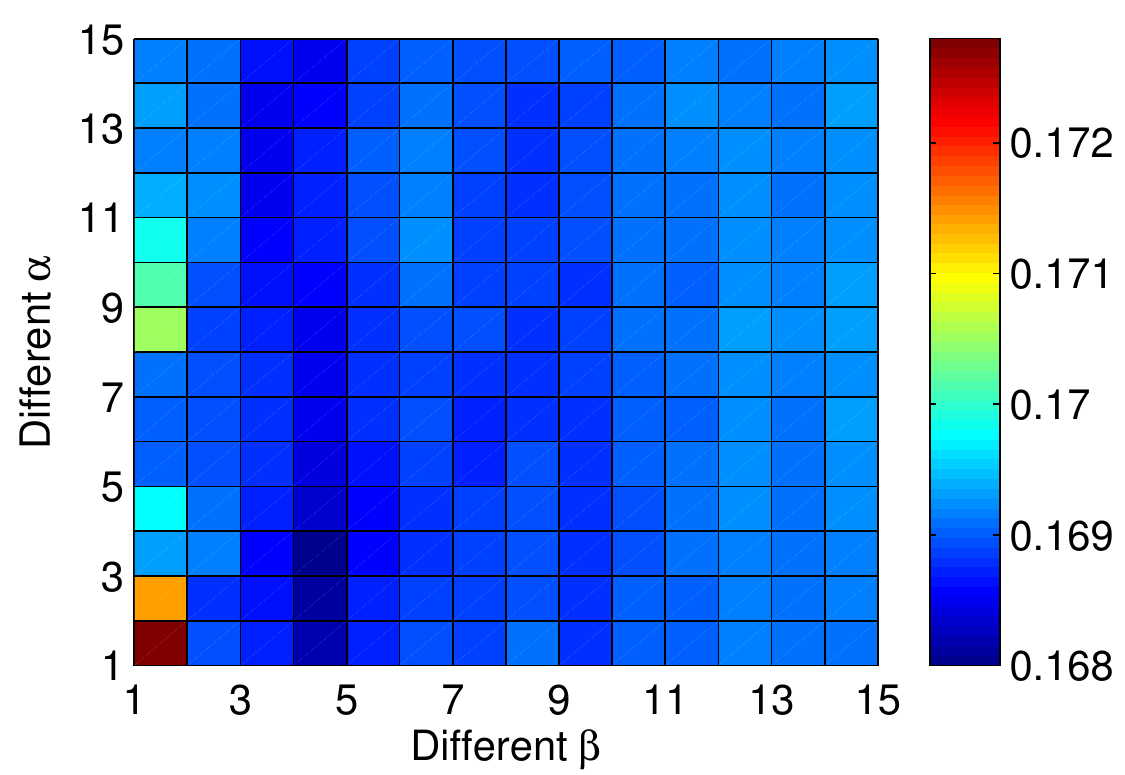}
\caption{\texttt{Hamming loss}}
\end{subfigure}%
\begin{subfigure}[t]{1.7in}
\centering
\includegraphics[width=1.7in]{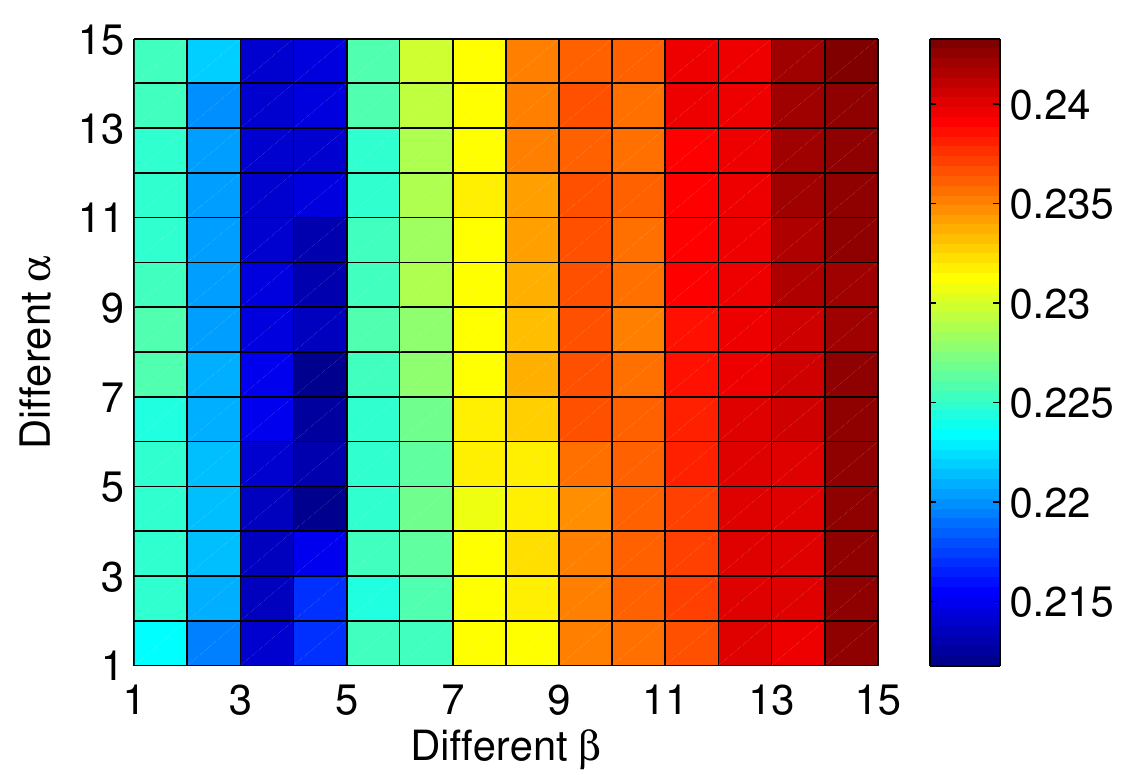}
\caption{\texttt{One-error}}
\end{subfigure}
\begin{subfigure}[t]{1.7in}
\centering
\includegraphics[width=1.7in]{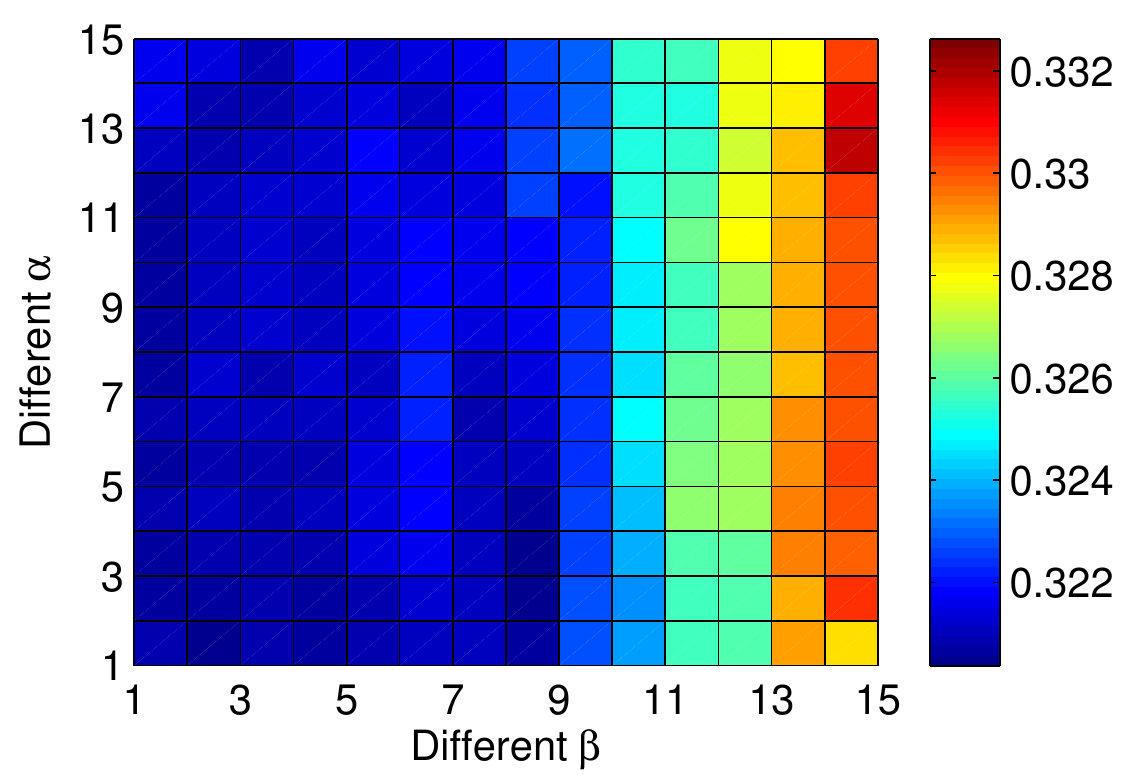}
\caption{\texttt{Coverage}}
\end{subfigure}%
\begin{subfigure}[t]{1.7in}
\centering
\includegraphics[width=1.7in]{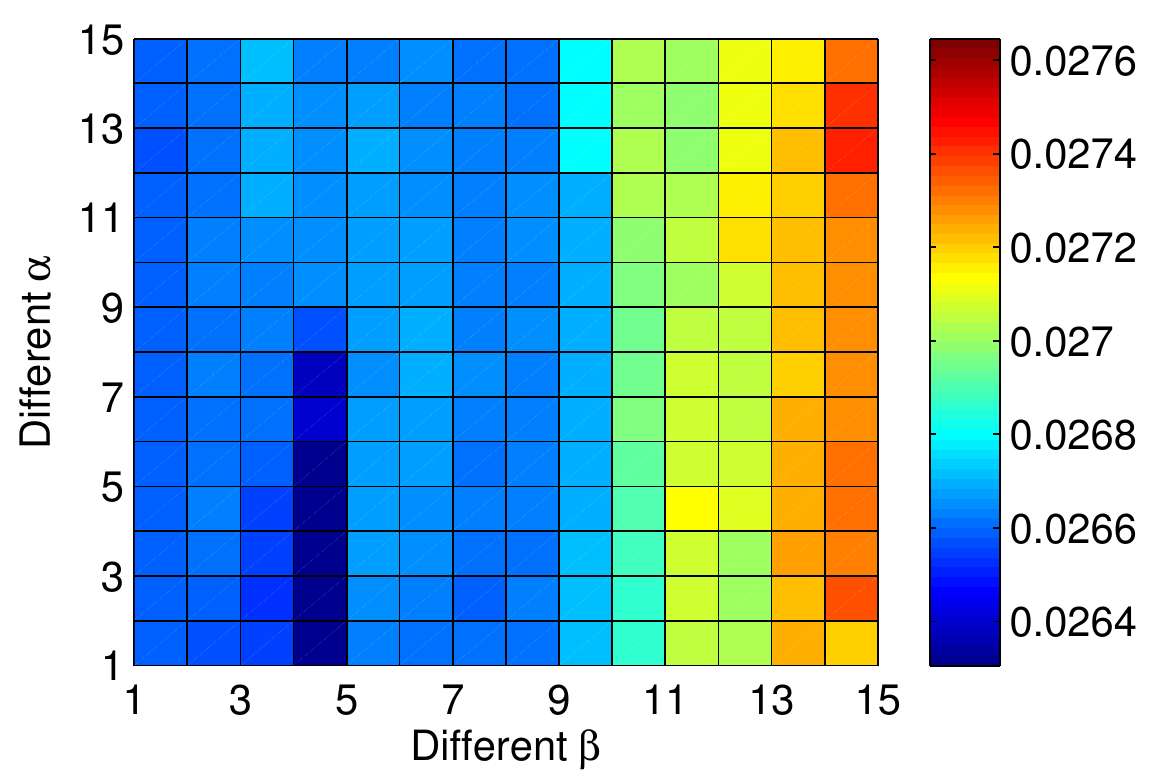}
\caption{\texttt{Ranking loss}}
\end{subfigure}%
\caption{The results of choosing different combinations of $\alpha$ and $\beta$. Top left to bottom right: the results of \texttt{Hamming loss}, \texttt{One-error}, \texttt{Coverage}, and \texttt{Ranking loss}.}\label{exp6}
\vspace{-0.2cm}
\end{figure*}



\subsection{Clinical Explanation and Discussion}
The evaluation metrics that we used so far (\ie, \texttt{Hamming loss}, \texttt{One-error}, \texttt{Coverage}, and \texttt{Ranking loss}) are originally designed to evaluate performance of multi-label prediction. We use these metrics as the reference from the data-driven perspective.

We now report the evaluation from the clinical perspective.
Specifically, we have asked three experienced \yh{neurologists} to double-check with the predicted results (\eg, the third prediction results in Figure \ref{exp4}). Specifically, the \yh{neurologists} are asked to determine whether the predicted drugs can improve the observed symptoms of a patient from the clinical perspective. A prescription prediction is said to be correct if more than 2 in 3 \yh{neurologists} agree it will improve PD, otherwise the prescription prediction is regarded as incorrect. This evaluation enables us to calculate the overall clinical accuracy as shown in Figure \ref{exp10}. For the prediction of 136 patients, PALAS achieves 87.5$\%$ accuracy (\ie, 119 correct and 17 incorrect), outperforms all the competitive methods (\ie, SMBR, BR, CLR, ML-$k$NN, Rank-SVM, MLMVL-MM and MVML).


In addition, we found that the incorrect prediction mostly happens for the prescription drug of non-motor symptoms. This could be due to the fact that, compared with motor-symptoms, the non-motor symptoms are more difficult to be accurately quantized.

\begin{figure}[htbp]
\centering
\includegraphics[width=2.8in]{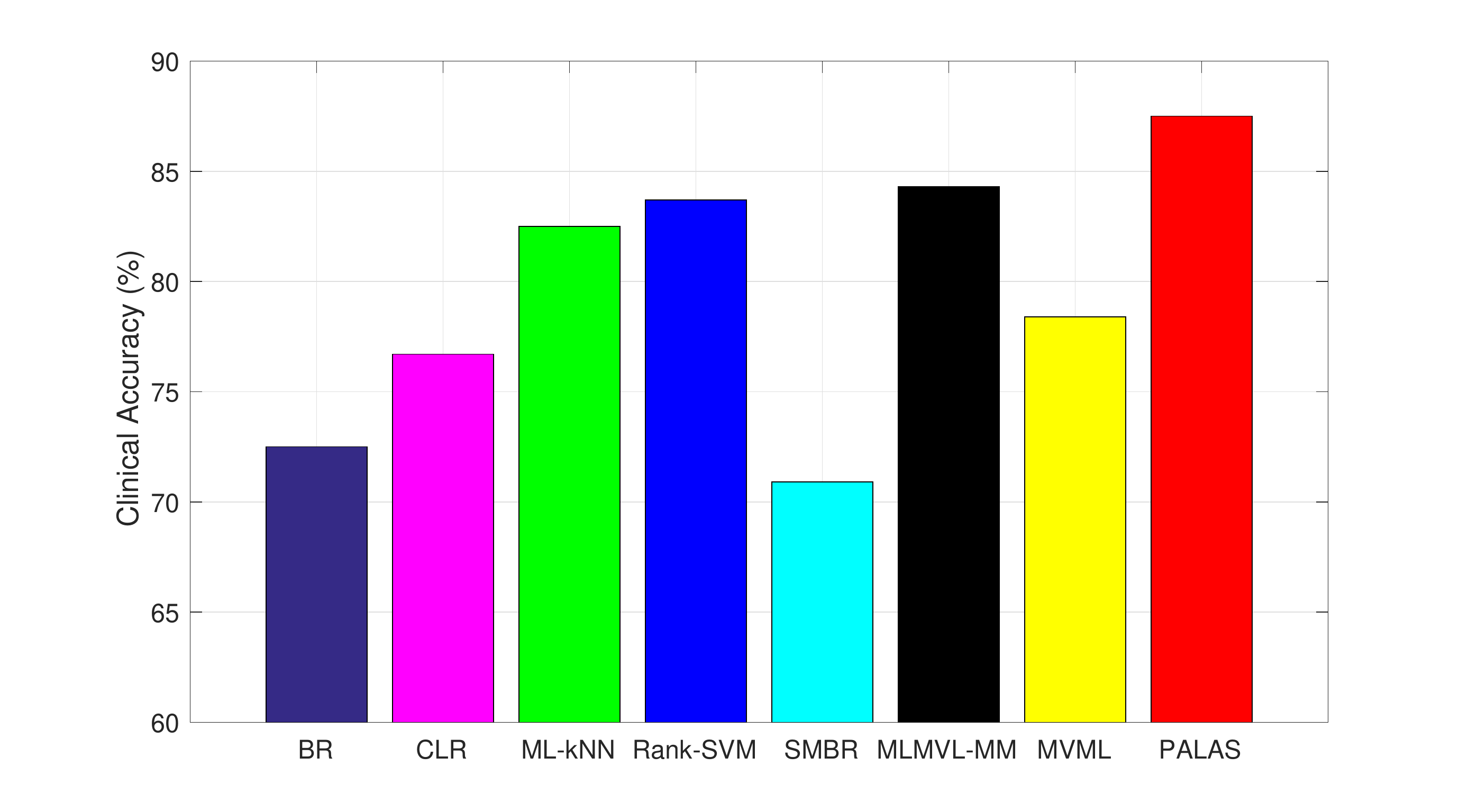}
\caption{The clinical accuracy of prescription prediction (evaluated by three \yh{neurologists}) using different methods.}\label{exp10}
\vspace{-0.3cm}
\end{figure}

\yh{For the limitation of our current study, firstly, the proposed method requires all the records (\ie, motor and non-motor symptoms) are completed. However, in some cases, we might encounter the data missing problem during recoding these symptoms. Secondly, for new coming data, we need to retrain our model by mixing the existed and new collected data. A better way is to design an online algorithm to update the learned model without retraining. Finally, a more interpretable way could be developed to enhance the current data-driven prediction. For these aforementioned issues, we will carefully consider in our future work. Also, in our future work, the threats-to-validity of our prediction results on a large population group will be designed and conducted.}

\yh{For the setting of supervised and transductive learning, from Table \ref{exp1}, we observe that the advantage of transductive setting is that it has a slight performance improvement with lower standard deviation compared with supervised setting.
However, its disadvantage comes from that for the new testing data, transductive setting is required to re-train the model. In this paper, we introduce these two settings to fully understand our method: 1) for the small sample-sized problem, transductive setting is preferable since re-training does not waste too much time and 2) for the large sample-sized problem, we will directly use supervised setting.}

\yhk{We wish to clarify that, whether or not the learned latent connections represent real causal relationships is beyond the consideration of the model/algorithm. That's exactly why the presence of experienced neurologists is still a must in practice. The model/algorithm should be regarded as an assistant to neurologists, recommending a good initial therapy from the data-driven perspective. It is still the neurologists' responsibility to treat the patients properly.}

\yh{For possible clinical application, the proposed method could be potentially employed in several scenario. The main scenario is that when the neurologists have disagreement on selection of drugs, the proposed method could provide a reference from data-driven perspective. Also, our method has its advantage to be extended to other important scenario: 1) A correlation study between patients and drugs can be further investigated to deeply understand different drug preference. 2) Different alternatives of drug recommendation from various aspects (by using different costs, \eg, number, price) can be generated together to aid neurologist during decision making.}


\section{Conclusion}
\label{sec:conclusion}
We study a novel problem in this paper, \ie, computer-aided prescription for PD, aiming to predict the suitable treatment drugs according to the observed motor and non-motor symptoms of a PD patient. The highlights of our work include: (1) our method is the first attempt to study the automatic prescription prediction for PD patient, (2) we adopt the multi-modality representation to incorporate various representations for better prediction performance, (3) the proposed PALAS model is able to capture the intrinsic symptom-to-drug relationship.
Experimental results validate the effectiveness of our method compared with other related baseline methods.


Our future directions include: (1) using wearable sensors to automatically record more information, (2) employing new modality (\eg, eye-tracker) to investigate personal habit of patients, and (3) designing an on-line learning strategy to incorporate new PD patients and the corresponding human-based prescription records, without re-training the current model.


\ifCLASSOPTIONcaptionsoff
  \newpage
\fi

\end{document}